\documentclass{article}

\usepackage[preprint]{neurips_2025}

\usepackage[utf8]{inputenc} 
\usepackage[T1]{fontenc}    
\usepackage{hyperref}       
\usepackage{url}            
\usepackage{booktabs}       
\usepackage{amsfonts}       
\usepackage{nicefrac}       
\usepackage{microtype}      
\usepackage[dvipsnames]{xcolor}         
\usepackage{amsthm}

\newtheorem{theorem}{Theorem}
\newtheorem{lemma}[theorem]{Lemma}
\newtheorem{corollary}[theorem]{Corollary}
\newtheorem{definition}[theorem]{Definition}
\newtheorem{example}[theorem]{Example}

\usepackage{algorithm}

\def\safedef#1{%
   \ifx#1\undefined
      \expandafter\def\expandafter#1%
   \else
      \errmessage{The \string#1 is defined already}%
      \expandafter\def\expandafter\tmp
   \fi
}

\usepackage{amsthm,amsmath,bbm,amsfonts,amssymb}
\usepackage{dsfont} 
\usepackage{mathtools}
\usepackage{commath}
\mathtoolsset{showonlyrefs=false}
\allowdisplaybreaks 

\usepackage{xspace}
\usepackage[T1]{fontenc}
\usepackage[utf8]{inputenc}
\usepackage[dvipsnames]{xcolor} 

\usepackage{hyperref,url}

\usepackage{wrapfig}
\usepackage{pbox} 
\usepackage{algorithm,algorithmic}

\usepackage{enumitem}
\usepackage[normalem]{ulem}

\newcommand{\kj}[1]{{\color{RedOrange}[#1]}}

{\color{kjsaved}}%

{\color{kjsavedenvkj}}%

\newcommand{\blue}[1]{{\color[rgb]{.3,.5,1}#1}}

\newcommand{\guide}[1]{{\color{Violet}[``#1'']}}

{\color{kjsavedrevised}}%
\usepackage{framed}
\usepackage[most]{tcolorbox}
\definecolor{kjgray}{rgb}{.7,.7,.7}

\newtheoremstyle{kjstyle}
{1ex} 
{\topsep} 
{\itshape} 
{} 
{\bfseries} 
{.} 
{.5em} 
{} 

\newtheoremstyle{kjstyle2}
{.0em} 
{.0em} 
{\itshape} 
{} 
{\bfseries} 
{.} 
{.5em} 
{} 

\newtheoremstyle{kjstylenoitalic}
{1ex} 
{\topsep} 
{} 
{} 
{\bfseries} 
{.} 
{.5em} 
{} 

\tcbset{kjboxstyle/.style={title={},breakable,colback=white,enhanced jigsaw,boxrule=1.3pt,sharp corners,colframe=kjgray,boxsep=0pt,coltitle={black},attach title to upper={},left=.8ex,bottom=.4em}}    

\tcbset{kjboxstylec/.style={title={},breakable,colback=white,enhanced jigsaw,boxrule=1.3pt,sharp corners,colframe=kjgray,boxsep=0pt,coltitle={black},attach title to upper={},before skip=1.5ex,left=.8ex,bottom=.4em,top=0.4em,enlarge top by=-0.3em,enlarge bottom by=-0.3em}}    

\usepackage{mdframed}
\usepackage{lipsum}
\definecolor{kjgray}{rgb}{.7,.7,.7}
\makeatletter



\makeatletter
\renewcommand{\paragraph}{%
  \@startsection{paragraph}{4}%
  {\z@}{0.50ex \@plus 1ex \@minus .2ex}{-1em}%
  {\normalfont\normalsize\bfseries}%
}
\makeatother

\usepackage{tabularx} 
\newcolumntype{P}[1]{>{\centering\arraybackslash}p{#1}}
\newcolumntype{M}[1]{>{\centering\arraybackslash}m{#1}}

\def\ddefloop#1{\ifx\ddefloop#1\else\ddef{#1}\expandafter\ddefloop\fi}

\def\ddef#1{\expandafter\def\csname #1#1\endcsname{\ensuremath{\mathbb{#1}}}}
\ddefloop ABCDFGHIJKLMNORSTUWXYZ\ddefloop 

\def\ddef#1{\expandafter\def\csname c#1\endcsname{\ensuremath{\mathcal{#1}}}}
\ddefloop ABCDEFGHIJKLMNOPQRSTUVWXYZ\ddefloop

\def\ddef#1{\expandafter\def\csname b#1\endcsname{\ensuremath{{\mathbf{#1}}}}}
\ddefloop ABCDEFGHIJKLMNOPQRSTUVWXYZ\ddefloop  
\def\ddef#1{\expandafter\def\csname b#1\endcsname{\ensuremath{{\boldsymbol{#1}}}}}
\ddefloop abcdeghijklmnopqrtsuvwxyz\ddefloop  

\def\ddef#1{\expandafter\def\csname h#1\endcsname{\ensuremath{\hat{#1}}}}
\ddefloop ABCDEFGHIJKLMNOPQRSTUVWXYZabcdefghijklmnopqrsuvwxyz\ddefloop 
\def\ddef#1{\expandafter\def\csname hc#1\endcsname{\ensuremath{\hat{\mathcal{#1}}}}}
\ddefloop ABCDEFGHIJKLMNOPQRSTUVWXYZ\ddefloop
\def\ddef#1{\expandafter\def\csname hb#1\endcsname{\ensuremath{\hat{\mathbf{#1}}}}}
\ddefloop ABCDEFGHIJKLMNOPQRSTUVWXYZ\ddefloop %
\def\ddef#1{\expandafter\def\csname hb#1\endcsname{\ensuremath{\hat{\boldsymbol{#1}}}}}
\ddefloop abcdefghijklmnopqrstuvwxyz\ddefloop %

\def\ddef#1{\expandafter\def\csname t#1\endcsname{\ensuremath{\tilde{#1}}}}
\ddefloop ABCDEFGHIJKLMNOPQRSTUVWXYZabcdefgijklmnpqtsuvwxyz\ddefloop 
\def\ddef#1{\expandafter\def\csname tc#1\endcsname{\ensuremath{\tilde{\mathcal{#1}}}}}
\ddefloop ABCDEFGHIJKLMNOPQRSTUVWXYZ\ddefloop
\def\ddef#1{\expandafter\def\csname tb#1\endcsname{\ensuremath{\tilde{\mathbf{#1}}}}}
\ddefloop ABCDEFGHIJKLMNOPQRSTUVWXYZ\ddefloop
\def\ddef#1{\expandafter\def\csname tb#1\endcsname{\ensuremath{\tilde{\boldsymbol{#1}}}}}
\ddefloop abcdefghijklmnopqrstuvwxyz\ddefloop %

\def\ddef#1{\expandafter\def\csname bar#1\endcsname{\ensuremath{\bar{#1}}}}
\ddefloop ABCDEFGHIJKLMNOPQRSTUVWXYZabcdefghijklmnopqrtsuvwxyz\ddefloop
\def\ddef#1{\expandafter\def\csname barc#1\endcsname{\ensuremath{\bar{\mathcal{#1}}}}}
\ddefloop ABCDEFGHIJKLMNOPQRSTUVWXYZ\ddefloop
\def\ddef#1{\expandafter\def\csname barb#1\endcsname{\ensuremath{\bar{\mathbf{#1}}}}}
\ddefloop ABCDEFGHIJKLMNOPQRSTUVWXYZ\ddefloop
\def\ddef#1{\expandafter\def\csname barb#1\endcsname{\ensuremath{\bar{\boldsymbol{#1}}}}}
\ddefloop abcdefghijklmnopqrstuvwxyz\ddefloop %

\def\ddef#1{\expandafter\def\csname war#1\endcsname{\ensuremath{\overline{#1}}}}
\ddefloop ABCDEFGHIJKLMNOPQRSTUVWXYZabcdefghijklmnopqrtsuvwxyz\ddefloop
\def\ddef#1{\expandafter\def\csname warc#1\endcsname{\ensuremath{\overline{\mathcal{#1}}}}}
\ddefloop ABCDEFGHIJKLMNOPQRSTUVWXYZ\ddefloop
\def\ddef#1{\expandafter\def\csname warb#1\endcsname{\ensuremath{\overline{\mathbf{#1}}}}}
\ddefloop ABCDEFGHIJKLMNOPQRSTUVWXYZ\ddefloop
\def\ddef#1{\expandafter\def\csname warb#1\endcsname{\ensuremath{\overline{\boldsymbol{#1}}}}}
\ddefloop abcdefghijklmnopqrstuvwxyz\ddefloop %

\def\sig{\sigma}

\def\dt{\delta}

\def\eps{\varepsilon}
\def\epsilon{\varepsilon}

\def\Dt{\Delta}

\usepackage{pgffor}
\def\greeksymbols{alpha,beta,gamma,gam,delta,dt,eps,epsilon,zeta,eta,theta,th,iota,kappa,kap,lambda,lam,mu,nu,xi,pi,rho,sigma,sig,tau,phi,chi,psi,omega,om,Gamma,Gam,Delta,Dt,Theta,Th,Lambda,Lam,Pi,Sigma,Sig,Phi,Psi,Omega,Om}
\def\greeksymbolsnoeta{alpha,beta,gamma,gam,delta,dt,eps,epsilon,zeta,theta,th,iota,kappa,kap,lambda,lam,mu,nu,xi,pi,rho,sigma,sig,tau,phi,chi,psi,omega,om,Gamma,Gam,Delta,Dt,Theta,Th,Lambda,Lam,Pi,Sigma,Sig,Phi,Psi,Omega,Om} 

\foreach \x in \greeksymbolsnoeta{\expandafter\xdef\csname b\x\endcsname{\noexpand\ensuremath{\noexpand\boldsymbol{\csname \x\endcsname}}}}

\foreach \x in \greeksymbols{\expandafter\xdef\csname h\x\endcsname{\noexpand\ensuremath{\noexpand\hat{\csname \x\endcsname}}}}
\foreach \x in \greeksymbolsnoeta{\expandafter\xdef\csname hb\x\endcsname{\noexpand\ensuremath{\noexpand\hat{\noexpand\boldsymbol{ \csname \x\endcsname}}}}}

\foreach \x in \greeksymbols{\expandafter\xdef\csname bar\x\endcsname{\noexpand\ensuremath{\noexpand\bar{\csname \x\endcsname}}}}
\foreach \x in \greeksymbolsnoeta{%
\expandafter\xdef\csname barb\x\endcsname{\noexpand\ensuremath{\noexpand\bar{\noexpand\boldsymbol{ \csname \x\endcsname}}}}
}

\foreach \x in \greeksymbols{\expandafter\xdef\csname t\x\endcsname{\noexpand\ensuremath{\noexpand\tilde{\csname \x\endcsname}}}}
\foreach \x in \greeksymbolsnoeta{\expandafter\xdef\csname tb\x\endcsname{\noexpand\ensuremath{\noexpand\tilde{\noexpand\boldsymbol{ \csname \x\endcsname}}}}}

\providecommand{\normz}[2][-1]{
\ensuremath{\mathinner{
\ifthenelse{\equal{#1}{-1}}{ 
\!\left\|#2\right\|}{}
\ifthenelse{\equal{#1}{0}}{ 
\|#2\|}{}
\ifthenelse{\equal{#1}{1}}{ 
\bigl\|#2\bigr\|}{}
\ifthenelse{\equal{#1}{2}}{ 
\Bigl\|#2\Bigr\|}{}
\ifthenelse{\equal{#1}{3}}{ 
\biggl\|#2\biggr\|}{}
\ifthenelse{\equal{#1}{4}}{ 
\Biggl\|#2\Biggr\|}{}
}} 
}  

\providecommand{\floor}[2][-1]{
\ensuremath{\mathinner{
\ifthenelse{\equal{#1}{-1}}{ 
\!\left\lfloor#2\right\rfloor}{}
\ifthenelse{\equal{#1}{0}}{ 
\lfloor#2\rfloor}{}
\ifthenelse{\equal{#1}{1}}{ 
\!\bigl\lfloor#2\bigr\rfloor}{}
\ifthenelse{\equal{#1}{2}}{ 
\!\Bigl\lfloor#2\Bigr\rfloor}{}
\ifthenelse{\equal{#1}{3}}{ 
\!\biggl\lfloor#2\biggr\rfloor}{}
\ifthenelse{\equal{#1}{4}}{ 
\!\Biggl\lfloor#2\Biggr\rfloor}{}
}} 
}

\providecommand{\ceil}[2][-1]{
\ensuremath{\mathinner{
\ifthenelse{\equal{#1}{-1}}{ 
\!\left\lceil#2\right\rceil}{}
\ifthenelse{\equal{#1}{0}}{ 
\lceil#2\rceil}{}
\ifthenelse{\equal{#1}{1}}{ 
\!\bigl\lceil#2\bigr\rceil}{}
\ifthenelse{\equal{#1}{2}}{ 
\!\Bigl\lceil#2\Bigr\rceil}{}
\ifthenelse{\equal{#1}{3}}{ 
\!\biggl\lceil#2\biggr\rceil}{}
\ifthenelse{\equal{#1}{4}}{ 
\!\Biggl\lceil#2\Biggr\rceil}{}
}} 
}

\newcommand{\fr}[2]{ { \frac{#1}{#2} }}

\newcommand{\wbar}[1]{{\ensuremath{\overline{#1}}}}  
\newcommand{\war}[1]{{\ensuremath{\overline{#1}}}} 

\newcommand{\T}{\top}
\def\wed{\wedge}
\def\tsty{\textstyle}

\def\cd{\cdot}

\def\larrow{\ensuremath{\leftarrow}} 
\def\rarrow{\ensuremath{\rightarrow}}

\definecolor{mygrn}{rgb}{0,.8,0}
\definecolor{myred}{rgb}{.8,0,0}

\DeclareMathOperator{\EE}{\mathbb{E}} 

\DeclareMathOperator{\PP}{\mathbb{P}}

\DeclareMathOperator{\Var}{{\mathrm{Var}}}

\DeclareMathOperator*{\argmin}{arg~min}

\DeclareMathOperator{\sign}{{\mathrm{sign}}}

\def\clip#1{\wbar{\del{#1}}}
\DeclareMathOperator{\one}{\mathds{1}\hspace{-.1em}}
\providecommand{\onec}[2][-1]{
\ensuremath{\mathinner{
\one\cbr[#1]{#2}
} 
}  
}

\DeclarePairedDelimiterX{\inp}[2]{\langle}{\rangle}{#1, #2}

\newcommand\declareop[3]{%
  \newcommand#1{%
    \mskip\muexpr\medmuskip*#2\relax
    {#3}%
    \mskip\muexpr\medmuskip*#2\relax
}}
\declareop\capprox{1}{{\sr{\const}{\approx}}} 
\declareop\logapprox{1}{{\sr{\mathrm{log}}{\approx}}} 

\newcommand{\lsim}{\mathop{}\!\lesssim}

\def\Bernoulli{{\ensuremath{\mathrm{Bernoulli}}}}

\def\const{\mathsf{const}}

\usepackage{pifont}
\newcommand{\sr}{\stackrel}

\makeatletter
\newcommand{\vast}{\bBigg@{3}}
\newcommand{\Vast}{\bBigg@{4}}
\makeatother

\newenvironment{talign*}
 {\csname align*\endcsname}
 {\endalign}

\def\chrulefill{\leavevmode\leaders\hrule height 0.7ex depth \dimexpr0.4pt-0.7ex\hfill\kern0pt}

\renewcommand{\cite}{\citep}

\usepackage[toc,page,header]{appendix}
\usepackage{minitoc}
\usepackage{silence}
\newif\ifFINAL

\FINALtrue 

\ifFINAL
  \def\blue#1{#1}
  \def\guide#1{}
  \def\kj#1{}
  \def\kjnew#1{}
\else
  \usepackage{transparent}
  \usepackage[inline]{showlabels}
  \usepackage{rotating}
  \renewcommand{\showlabelfont}%
  {\transparent{0.8}\scriptsize\bf\slshape\color{Lavender}}
\fi

\def\LHS{{\text{LHS}}}

\def\warcF{{\war\cF}}

\def\warDt{{\war\Dt}}

\title{Second-Order Bounds for [0,1]-Valued Regression via Betting Loss}

\author{%
  Yinan Li \\
  Department of Computer Science\\
  University of Arizona\\
  \texttt{yinanli@arizona.edu} \\
  \And
  Kwang-Sung Jun \\
  Department of Computer Science\\
  University of Arizona\\
  \texttt{kjun@cs.arizona.edu} \\
}

\begin{document}
\doparttoc 
\faketableofcontents 

\maketitle


\begin{abstract}
We consider the $[0,1]$-valued regression problem in the i.i.d. setting. 
In a related problem called cost-sensitive classification, \citet{foster21efficient} have shown that the log loss minimizer achieves an improved generalization bound compared to that of the squared loss minimizer in the sense that the bound scales with the cost of the best classifier, which can be arbitrarily small depending on the problem at hand.
Such a result is often called a first-order bound.
For $[0,1]$-valued regression, we first show that the log loss minimizer leads to a similar first-order bound.
We then ask if there exists a loss function that achieves a variance-dependent bound (also known as a second order bound), which is a strict improvement upon first-order bounds.
We answer this question in the affirmative by proposing a novel loss function called the betting loss.
Our result is ``variance-adaptive'' in the sense that the bound is attained \textit{without any knowledge about the variance}, which is in contrast to modeling label (or reward) variance or the label distribution itself explicitly as part of the function class such as distributional reinforcement learning.
\end{abstract}

\section{Introduction}
\label{sec:intro}
We consider the $[0,1]$-valued regression problem.
In this task, we are given a dataset $\blue{D_n} = \{ (x_t, y_t)\}_{t=1}^n$ where $\blue{x_t}\in \cX$ is the feature of the $t$-th data point and $\blue{y_t}\in[0,1]$ is its label.
We assume the data $(x_t,y_t) \sim \cD$ is i.i.d., $\forall t\in [n]$.
The goal is to, given a function class $\cF \subset \{\cX \rarrow [0,1]\}$, find a function $f$ such that the prediction $f(x)$ is as close as possible to $y$ on average where $(x,y) \sim \cD$.

While being one of the simplest machine learning tasks, this regression task applies to numerous practical applications.
First of all, classification is a special case of this problem where the label space is $\cY = \{0,1\}$.
Second, in Reinforcement Learning (RL), the rewards are typically bounded, and when the episode length is upperbounded, the cumulative reward per episode is also bounded.
Thus, in the function approximation setting, one can easily scale the cumulative rewards from each state-action to $[0,1]$ and perform regression.
With this regression, one can construct a policy that choose the action with the highest predicted value.
In goal-oriented RL, regardless of the length of the episode, the rewards are given only at the end of the episode, so, as long as the reward is bounded in a fixed interval, $[0,1]$-valued regression applies.
Finally, human preferences can mostly be expressed as a value in $[0,1]$.
For example, 5-star ratings ($\in \{1,2,3,4,5\}$) for products can be affine-transformed to $[0,1]$.
Furthermore, datasets commonly used for aligning Large Language Models (LLMs) such as HelpSteer2 originally contain scores $\{0,1,2,3,4\}$~\cite{wang24helpsteer2}.
Therefore, despite being simple and rather elementary, $[0,1]$-valued regression is still important, and theoretical and algorithmic advancements may have a huge impact in practice.

What do we know about the fundamental performance limits of $[0,1]$-valued regression?
The de facto standard regression algorithm is to simply minimize the squared loss.
However, it is not clear at all if squared loss is optimal for $[0,1]$-valued regression.
In a related problem called cost-sensitive classification, \citet{foster21efficient} have shown that the squared loss is not optimal for $[0,1]$-valued costs. 
Instead, they have shown that the log loss achieves a strictly improved performance bound, a rate that is provably not attained by the squared loss~\cite[Theorem 2]{foster21efficient}.
Specifically, their bound is of the \textit{first-order} type, which means that the performance bound scales with the \textit{magnitude} of the cost/reward being accumulated by the optimal policy. 
Such a bound is never worse than the standard worst-case bound, yet can be much smaller depending on the problem at hand. 
This has also been called small-loss bound and can be viewed as a problem-dependent accelerated rate.

%
Such a first-order bound appeared in various machine learning problems~\cite{freund97decision,foster21efficient,wagenmaker22first}.
In these problems, there is another concept called \textit{second-order} bound~\cite{cb07improved}.
While the precise definition can vary across problems, when making stochastic assumptions about how $y$ is related to $x$, it means that the bound scales with the label's second moment or the variance, which can be much smaller than the magnitude of the label.
We elaborate more on this in Section~\ref{sec:related}.



%

Motivated by the fact that \citet{foster21efficient} simply switched a loss function to obtain a first-order bound in cost-sensitive classification, we first report that the same is true in $[0,1]$-valued regression (see Theorem~\ref{thm:first-order-bound} in Section~\ref{sec:preliminaries}).
Given this positive answer, we take a step further and ask the following research question:
%
%
\begin{center}
  \begin{tabular}{c}
    \textit{Does there exist a loss function whose minimizer leads to a second-order bound?}
    \\
  \end{tabular}
\end{center}
In this paper, we provide an affirmative answer by proposing a novel loss function inspired by the betting-based confidence set~\cite{waudby23estimating,orabona24tight}.
We emphasize that our algorithm does not require conditional variances as input and allows them to be arbitrarily different.
This is in stark contrast to some existing work that either requires the variance as input~\cite{zhao23optimal} or models variance as part of function approximation~\cite{wang24more,weltz23experimental}.
In some sense, our result shows that obtaining second-order bounds (i.e., adapting to variances) is a free lunch, statistically speaking, in the sense that we do not have to model variance to adapt to it.
While there are works that achieve second-order bounds without knowledge of the conditional variances~\cite{zhao24adaptive,jun24noise,jia24how,pacchiano25second}, the tools therein are specialized for their own contextual bandit problem and do not naturally imply an estimator for $[0,1]$-valued regression. 
We discuss more related work in Section~\ref{sec:related}.


\section{Preliminaries}
\label{sec:preliminaries}
\textbf{Notations.}
We denote $f_x := f(x)$ for any function $f$ and any $x \in \cX$. 
We take the nonasymptotic version of $\lsim$; i.e., $f(x) \lsim g(x)$ means $\exists c > 0$, s.t. $f(x) \leq c \cd g(x), ~\forall x$. 

\textbf{Regression with $[0,1]$-valued label.}
We consider the standard supervised learning setting with bounded regression targets. Let $\cX$  denote the input space. We observe a dataset $D_n = \cbr{(x_t, y_t)}_{t=1}^n$ where each pair $(x_t, y_t)$ is drawn i.i.d. from an unknown distribution $\cD_{X,Y}$ over $\cX \times [0,1]$. 
We denote by $\cD_{Y|X}$ the distribution of the label conditioning on the input, and $\cD_{X}$ the marginal distribution of the input.

Let $\cF \subset \cbr{\cX \rightarrow [0,1]}$ be a class of prediction functions mapping inputs to the unit interval. We assume \emph{realizability}, i.e., there exists a function $f^* \in \cF$ such that: 
\[
  \EE_{y\sim \cD_{Y|X}}[y \mid x] = f^*(x), \text{ for all } x \in \cX.
\] 
Note that, since we do not have further restrictions on $\cD_{X,Y}$, the conditional variance $\sig_x^2 := \EE_{y\sim \cD_{Y|X}} [(y-f^*(x))^2 \mid x]$ can be quite different across $x\in \cX$.

Given the observed data $D_n$, the learning goal is to find a hypothesis $\hf \in \cF$ that achieves low expected absolute error with respect to the ground-truth regression function $f^*$: \[
\EE_{x \sim \cD_X} \sbr{ |f^*(x) - \hf(x)  | }. 
\]

A central goal in statistical learning theory is to bound such a generalization error of the learned hypothesis in terms of the sample size $n$, the function class complexity $\ln|\cF|$, and the confidence level $\dt$. 

One popular algorithm for regression is the squared loss minimizer:
\begin{align}
    \hat f = \arg \min_{f\in\cF} \fr1n \sum_{(x,y)\in D_n} \fr12 (f(x) - y)^2
\end{align}

A \emph{classical result} on the squared loss minimizer yields the following: 
\[
\EE_{x \sim \cD_X} \sbr{ |f^*(x) - \hf(x)  |}
\lsim \sqrt{\frac{\ln (|\cF|/ \dt)}{n}} . 
\] 

While this bound is simple and general, it does not incorporate any notion of conditional variance. 
It treats all inputs as equally noisy, making it inherently variance-insensitive and potentially loose in heterogeneous noise settings.




A recent result on the log loss minimizer \cite[Theorem 3]{foster21efficient} immediately implies the following first-order generalization bound, which scales with the magnitude of the target regression function $f^*(x)$ and its complement $1-f^*(x)$ in expectation:  
\begin{align}
    \EE_{x \sim \cD_X} \sbr{ |f^*(x) - \hf(x)  | }
\lsim \sqrt{\frac{\del{\EE_x[f^*(x)] \wed \EE_x[1-f^*(x)]} \cd \ln (|\cF|/ \dt)}{n}} + \frac{\ln (|\cF|/ \dt)}{n}. 
\label{eqn:foster21-first-order-bound}
\end{align}

In the following theorem, we further improve the bound above to scale with $\EE_x[f^*(x)(1-f^*(x))]$. 

\begin{theorem} \label{thm:first-order-bound}
With probability at least $1-\delta$, 
     \[
\EE_{x \sim \cD_X} \sbr{ |f^*(x) - \hf(x)  | }
\lsim \sqrt{\frac{\EE_x[f^*(x)(1-f^*(x))] \cd \ln (|\cF|/ \dt)}{n}} + \frac{\ln (|\cF|/ \dt)}{n}. 
\]
\end{theorem}
We note that this bound strictly improves upon the immediate implication of~\citet{foster21efficient} (Eqn.~\eqref{eqn:foster21-first-order-bound}), as $\EE_x[f^*(x) \wedge (1-f^*(x))] \leq \EE_x[f^*(x)] \wed \EE_x[1-f^*(x)]$, the gap between these two quantities can be arbitrarily large as we show in Appendix~\ref{sec:comparison-first-order}. 

The bound in Theorem~\ref{thm:first-order-bound} depends on the worst-case proxy $f^*(x)(1-f^*(x))$, which upper bounds the conditional variance $\sig_x^2 = \EE [(y-f^*(x))^2 \mid x]$, as indicated by the following Lemma. 
\begin{lemma}
    Let $Y \in [0,1]$ be a random variable. Then $\Var(Y) \leq \EE[Y(1-Y)]$, and the equality is attained iff $Y$ is Bernoulli-distributed. 
    \label{lem:var-ub}
\end{lemma}

However, this proxy can be loose, especially when $y \mid x$ is not Bernoulli (if Bernoulli, then $\sig^2_x = f^*(x)(1-f^*(x))$). 
In modern applications, such as those mentioned in Section~\ref{sec:intro},  
the labels are often heteroscedastic and non-Bernoulli: some inputs yield highly confident predictions with low variance, while others are more uncertain, and the outputs can take values other than the boundary points 0 and 1.  \kj{[x] how about we simply refer to the examples we showed in the intro? e.g., we can say like ``as we mentioend in Section `'' } 
In such settings, first-order bounds may fail to capture the true learnability of the problem.

To address this, our objective in this work is to derive \emph{second-order generalization bounds} that adapt to the true conditional variance. Specifically, we aim to obtain bounds of the form:  \[
\EE_{x \sim \cD_X} \sbr{ |f^*(x) - \hf(x)  | }
\lsim \sqrt{\frac{\EE_x[\sig_x^2] \cd \ln (|\cF|/ \dt)}{n}} + \frac{\ln (|\cF|/ \dt)}{n}, 
\]
which provides tighter guarantees in settings where $\cD_{X}$ places a nontrivial probability on $x$ such that the conditional variance $\sig_x^2$ is much smaller than the worst-case upper bound $f^*(x)(1-f^*(x))$.

\section{Second-Order Bound via Betting Loss}
We propose a new regression algorithm that adapts to conditional variance by minimizing a worst-case form of betting loss, a loss function originally motivated by coin-betting frameworks. The goal is to learn a hypothesis $\hat{f}$ from a function class $\mathcal{F}$, using a dataset $D_n = \{(x_t, y_t)\}_{t=1}^n$ of feature-label pairs, such that $\hat{f}$ achieves strong generalization guarantees that adapt to the heteroskedastic noise structure of the problem.

At the core of our algorithm is a robust min-max optimization that seeks a function $f \in \mathcal{F}$ whose performance is stable against perturbations in the direction of any other hypothesis $h \in \mathcal{F}$. To this end, we define:
\begin{itemize}
    \item A fixed parameter $ \war\phi := \frac{n}{4}$ that controls the magnitude of perturbation.
    \item A clipped betting loss function: \[
     H_{\phi,c}(h, f) := \sum_{(x,y) \in D_n} \ln \del{ 1+ (y-f_x) \clip{\phi(h_x-f_x)}_{[-c, c]} } 
    \]
    where $\overline{(x)}_{[a,b]} := \max\{\min\{x,b\}, a\}$ and $c \in [0, \tfrac{1}{4}]$ is a clipping threshold. 
\end{itemize}
This formulation ensures that $\hf$ performs well even under worst-case (clipped) perturbations in the direction of any other hypothesis $h \in \mathcal{F}$, across all allowed magnitudes $\phi$ and clipping levels $c$. Algorithm~\ref{alg:betting-regression} summarizes the procedure.

\begin{algorithm}[H]
\caption{Variance-Adaptive Regression via Betting Loss Minimization}
\label{alg:betting-regression}
\begin{algorithmic}[1]
\REQUIRE Dataset $ D_n = \{(x_t, y_t)\}_{t=1}^n $, hypothesis class $\cF$
\STATE Compute the output:
\[
\hat{f} := \arg\min_{f \in \cF} \max_{h \in \cF} \max_{\phi \in [0,\war\phi]} \max_{c \in [0, \tfrac{1}{4}]} \frac{1}{n} H_{\phi,c}(h, f)
\]
\RETURN $\hat{f}$
\end{algorithmic}
\end{algorithm}

The objective minimized by Algorithm~\ref{alg:betting-regression} -- the worst-case clipped betting loss -- directly governs the generalization behavior of the learned hypothesis.
To formalize this, define the functional: 
\begin{align*}
  L(f) := \max_{h \in \cF} \max_{\phi \in [0,\war\phi]} \max_{c \in [0,\fr14]} \fr 1n \sum_{(x,y) \in D_n} \ln \del{ 1+ (y-f_x) \clip{\phi(h_x-f_x)}_{[-c, c]} } 
\end{align*}
which exactly matches the objective minimized by the algorithm. The following theorem provides a high-probability bound on the expected absolute error of any $f \in \cF$ in terms of this loss.

\begin{theorem}
There exists numerical constants $c_1, c_2$ and $c_3$, such that with probability at least $1-\dt$, 
\begin{align*}
  \forall f \in \cF, ~~\EE_x |f_x - f^*_x| 
  &\le  c_1 \cd \sqrt{  \EE \sig_{x}^2 \cd \del{\fr1n\ln\del{\fr{ |\cF| n}{\dt}} + \max\{L(f) - L(f^*), 0\} }  } 
  \\ &+   c_2 \cd \fr1n\ln\del{\fr{ |\cF| n}{\dt}} + c_3 \cd (L(f) - L(f^*) )
\end{align*}
\label{thm:second-order-bound}
\end{theorem}

This result provides a high-probability bound on the prediction error $\EE_x |f_x - f^*_x|$ of any $f \in \cF$, in terms of the difference in empirical betting loss $L(f) - L(f^*)$. Crucially, the bound adapts to the conditional variance $\sigma_x^2$ in the leading term. The excess betting loss $L(f) - L(f^*)$ directly controls the mean absolute error. In particular, applying the theorem to the output $\hf = \arg\min_{f \in \cF} L(f)$ of Algorithm~\ref{alg:betting-regression}, we obtain: 
\begin{align*}
  \EE_x |\hf_x - f^*_x| 
  &\le  c_1 \cd \sqrt{  \EE \sig_{x}^2 \cd \del{\fr1n\ln\del{\fr{ |\cF| n}{\dt}} }  } 
  +   c_2 \cd \fr1n\ln\del{\fr{ |\cF| n}{\dt}}
\end{align*}
This bound reflects a variance-adaptive fast rate, which improves over convergence bounds that scale with the worst-case noise $f^*_x(1-f^*_x)$. In particular, when the conditional variance $\sigma_x^2$ is small on average, the mean absolute error of $\hf$ becomes correspondingly small -- without requiring prior knowledge of the noise structure. This establishes Algorithm~\ref{alg:betting-regression} as a variance-adaptive learning procedure for [0,1]-valued regression. 

Theorem~\ref{thm:second-order-bound} provides a general excess risk bound that holds uniformly for all $f \in \cF$ over any finite hypothesis class $\cF$. 
Moreover, by combining this result with complexity control via covering numbers, we can derive concrete generalization bounds for standard function classes. The following corollary instantiates this extension for the linear class:

\begin{corollary}[Linear class]
Let $\cF$ be a linear function class in $d$-dimensional space: $\cF = \{x \mapsto x^\T\theta+\fr12: \normz{\theta}_2 \le \fr12\}$ and $\cX$ be the instance space: $\cX = \{x\in \RR^d: \normz{x}_2 \le 1\}$. Then, 
there exists constants $c_1$ and $c_2$, such that with probability at least $1-\dt$, the output $\hf = \arg\min_{f \in \cF} L(f)$ satisfies: 
\begin{align*}
    \EE_x |\hf_x - f^*_x|
    &\le c_1 \cd \sqrt{ \EE_x\sig_{x}^2 \fr dn \ln (\fr{n}{\dt}) } + c_2 \cd \fr dn \ln (\fr{n}{\dt}). 
\end{align*}
\label{cor:linear-class}
\end{corollary}
This result follows from Theorem~\ref{thm:second-order-bound} by applying standard metric entropy for linear classes w.r.t $\|\cd\|_\infty$, which scales as $d \log (\fr 1 \epsilon)$. We present the full proof in the appendix. 

The dominant term in our bound reflects the minimax-optimal rate for linear regression in 
$d$-dimensional spaces. This aligns with classical results showing that, even in the well-specified setting with Gaussian noise, no estimator can achieve a faster worst-case rate than $\mathcal{O} (\sqrt{\fr dn})$ for $\EE_x |\hf_x - f^*_x|$ prediction  error~\citep{tsybakov2004introduction, wainwright19high}, matching our bound up to logarithmic factors.  


This corollary illustrates the generality of the proposed framework. Notably, our bound is adaptive to the conditional variance of the noise while matching the worst-case guarantees of classical approaches for the linear class.  While the algorithm itself minimizes an empirical objective, the generalization theory built on the betting loss applies to any $f \in \cF$, and can be extended to infinite hypothesis classes using standard covering number techniques. 
\section{Related Work}
\label{sec:related}


\textbf{Regression with heteroscedastic noise.}
Regression with heteroscedastic noise can be dated back to~\cite{aitkin35least}, and has been further developed into Gaussian processes~\cite{kersting07most,goldberg97regression}.
However, these works assume knowledge of the variances.

\textbf{First- and second-order bounds.}
From the adversarial setting, to our knowledge, the first appearance of the first-order bound is from the prediction with expert advice setting~\citet{freund97decision}, which is an adversarial (i.e., nonstochastic) setting.
In the same setting, \citet{cb07improved} developed a second-order bound with the prod algorithm.
Note that the notion of second-order can be defined in various ways; e.g., \citet{hazan10extracting}.
In $K$-armed bandits, \citet{stoltz05incomplete} and \citet{allenberg06hannan} have shown first-order bound.
In linear bandits, obtaining a first-order regret was an open problem \cite{agarwal17open}, which was later resolved by~\citet{az18make}.
Second-order bounds were developed by \citet{hazan11better} and improved by \citet{ito20tight}.
We refer to~\cite{neututorial} for a review of the first-/second-order bounds in adversarial settings.

\textbf{Stochastic bandits with function approximation.}
We now discuss first-/second-order bounds in the stochastic bandit problem with function approximation (also known as structured bandits).
Hereafter, unless noted otherwise, the noise model is such that the reward (label) is bounded with a known range, which can be easily translated to $[0,1]$-valued reward.
The first-order bound was first obtained by \citet{foster21efficient} for generic function classes.
We classify second-order bounds as follows:
\begin{itemize}
  \item \textbf{With known variance}: 
      Based on weighted linear regression, \citet{zhou21nearly,zhou22computationally,zhao23optimal} have obtained second-order bounds in linear models.
  \item \textbf{Unknown variances but with models of variance or distribution}:
      In the pure exploration setting, \citet{weltz23experimental} have considered modeling the variance explicitly with a specific function class in order to obtain improved sample complexity.
      \citet{wang24more} have shown that modeling not just mean or variance but the noise distribution itself leads to a second-order bound. 
      However, note that modeling variance or distribution has a price to pay due to the extra modeling.
  \item \textbf{Unknown variances}:
      The last set of works do not make any effort in modeling the variance or distribution, and thus there is no extra price to pay, at least in the statistical sense.
      For the linear model, \citet{zhang22optimal} proposed a second-order regret bound, which was further improved by \citet{kim22improved}.
      The optimal rate in this setting was first obtained by \citet{zhao23variance}, and \citet{jun24noise} obtained the same bound but with improved numerical performance along with removal of an unnatural technical assumption on the noise.
      For generic function class, \citet{jia24how} and \citet{pacchiano25second} both independently developed a second-order bound where the dependence of the function class appears as the eluder dimension~\cite{russo13eluder}.
\end{itemize}
While the work with unknown variances are the closest to our work, we emphasize that the tools developed therein do not directly imply any meaningful result for the regression setting, to our knowledge. 
Delineating the challenges is left as future work.
That said, we believe the estimator might be useful in obtaining an improved second-order regret bound in bandits with general function classes, just in the same way that the log loss has played a role in obtaining a first-order bound \cite{foster21efficient}.

There is another set of work that considers sub-Gaussian noise, which is more general than the bounded reward. 
\citet{kirschner18information} consider the heteroscedastic noise in linear bandits for the first time, to our knowledge.
Their work assumes that the noise is $\sigma^{2}(x)$-sub-Gaussian when pulling arm $x$ and that the value of $\sigma^2(x)$ is known to the algorithm.
\citet{jun24noise} considered a further generalized setting where the noise is $\sigma^2_t$-sub-Gaussian at time step $t$, and $\sigma^2_t$ can be dependent on anything that happened up to choosing the arm $x_t$ at time $t$.
Furthermore, they assume that the algorithm does not have access to $\sig^2_t$ but rather an upper bound $\sig^2_0$ and have shown that there exists a computationally efficient algorithm whose performance provably adapts to $\max_t \sig^2_t$ for the leading term (though there is a lower order term with a $\sig^2_0$ dependence).

\section{Conclusion}

This paper introduces a new approach to regression that achieves second-order generalization guarantees by minimizing a novel \emph{betting loss} function inspired by the betting-based confidence bounds.
Our analysis establishes that minimizing this loss yields estimators whose guarantee adapts to the conditional variance of the data -- without requiring any prior knowledge. 
Our bound is first-of-its-kind, to our knowledge.

We further demonstrate that our generalization error bounds scale favorably with the local noise level and apply broadly across both finite and infinite hypothesis classes, with a concrete instantiation for the linear function class. These results show that, under suitable conditions, variance adaptivity can be attained ``for free'' in the statistical sense, through a carefully chosen loss function alone and without adding extra assumptions. 
Despite its current computational challenges, the betting loss offers a principled framework with strong theoretical guarantees, particularly in capturing variance-adaptive behavior. Its foundational role in advancing our understanding of adaptive learning justifies further investigation, potentially inspiring new algorithmic approaches or practical surrogates.

Our results demonstrate that variance-aware learning can be achieved through the design of the loss function itself -- without requiring variance estimation or modeling. This insight suggests several promising directions for extending the betting loss framework to other domains where adapting to noise is critical, such as active learning and exploration in reinforcement learning.


\clearpage

\bibliographystyle{abbrvnat_lastname_first}
\bibliography{library-overleaf}

\newcommand{\noop}[1]{}
\begin{thebibliography}{35}
\providecommand{\natexlab}[1]{#1}
\providecommand{\url}[1]{\texttt{#1}}
\expandafter\ifx\csname urlstyle\endcsname\relax
  \providecommand{\doi}[1]{doi: #1}\else
  \providecommand{\doi}{doi: \begingroup \urlstyle{rm}\Url}\fi

\bibitem[Agarwal et~al.(2017)Agarwal, Krishnamurthy, Langford, Luo, et~al.]{agarwal17open}
Agarwal, A., Krishnamurthy, A., Langford, J., Luo, H., et~al.
\newblock Open problem: First-order regret bounds for contextual bandits.
\newblock In \emph{Proceedings of the Conference on Learning Theory (COLT)}, pages 4--7, 2017.

\bibitem[Aitkin(1935)]{aitkin35least}
Aitkin, A.
\newblock On least squares and linear combination of observations.
\newblock \emph{Proceedings of the Royal Society of Edinburgh}, 55:\penalty0 42--48, 1935.

\bibitem[Allen-Zhu et~al.(2018)Allen-Zhu, Bubeck, and Li]{az18make}
Allen-Zhu, Z., Bubeck, S., and Li, Y.
\newblock Make the minority great again: First-order regret bound for contextual bandits.
\newblock In \emph{Proceedings of the International Conference on Machine Learning (ICML)}, pages 186--194, 2018.

\bibitem[Allenberg et~al.(2006)Allenberg, Auer, Gy{\"{o}}rfi, and Ottucs{\'{a}}k]{allenberg06hannan}
Allenberg, C., Auer, P., Gy{\"{o}}rfi, L., and Ottucs{\'{a}}k, G.
\newblock {Hannan Consistency in On-Line Learning in Case of Unbounded Losses Under Partial Monitoring}.
\newblock In \emph{Algorithmic Learning Theory (ALT)}, pages 229--243. 2006.

\bibitem[Cesa-Bianchi et~al.(2007)Cesa-Bianchi, Mansour, and Stoltz]{cb07improved}
Cesa-Bianchi, N., Mansour, Y., and Stoltz, G.
\newblock Improved second-order bounds for prediction with expert advice.
\newblock \emph{Machine Learning}, 66:\penalty0 321--352, 2007.

\bibitem[Foster and Krishnamurthy(2021)]{foster21efficient}
Foster, D.~J. and Krishnamurthy, A.
\newblock Efficient first-order contextual bandits: Prediction, allocation, and triangular discrimination.
\newblock \emph{Advances in Neural Information Processing Systems (NeurIPS)}, pages 18907--18919, 2021.

\bibitem[Freund and Schapire(1997)]{freund97decision}
Freund, Y. and Schapire, R.~E.
\newblock A decision-theoretic generalization of on-line learning and an application to boosting.
\newblock \emph{Journal of computer and system sciences}, 55\penalty0 (1):\penalty0 119--139, 1997.

\bibitem[Goldberg et~al.(1997)Goldberg, Williams, and Bishop]{goldberg97regression}
Goldberg, P., Williams, C., and Bishop, C.
\newblock Regression with input-dependent noise: A gaussian process treatment.
\newblock \emph{Advances in Neural Information Processing Systems (NeurIPS)}, 1997.

\bibitem[Hazan and Kale(2010)]{hazan10extracting}
Hazan, E. and Kale, S.
\newblock Extracting certainty from uncertainty: Regret bounded by variation in costs.
\newblock \emph{Machine learning}, 80:\penalty0 165--188, 2010.

\bibitem[Hazan and Kale(2011)]{hazan11better}
Hazan, E. and Kale, S.
\newblock Better algorithms for benign bandits.
\newblock \emph{Journal of Machine Learning Research}, 12\penalty0 (4), 2011.

\bibitem[Ito et~al.(2020)Ito, Hirahara, Soma, and Yoshida]{ito20tight}
Ito, S., Hirahara, S., Soma, T., and Yoshida, Y.
\newblock Tight first- and second-order regret bounds for adversarial linear bandits.
\newblock In \emph{Advances in Neural Information Processing Systems (NeurIPS)}, volume~33, pages 2028--2038, 2020.

\bibitem[Jia et~al.(2024)Jia, Qian, Rakhlin, and Wei]{jia24how}
Jia, Z., Qian, J., Rakhlin, A., and Wei, C.-Y.
\newblock How does variance shape the regret in contextual bandits?
\newblock In \emph{Advances in Neural Information Processing Systems (NeurIPS)}, 2024.

\bibitem[Jun and Kim(2024)]{jun24noise}
Jun, K.-S. and Kim, J.
\newblock Noise-adaptive confidence sets for linear bandits and application to bayesian optimization.
\newblock In \emph{Proceedings of the International Conference on Machine Learning (ICML)}, 2024.

\bibitem[Kersting et~al.(2007)Kersting, Plagemann, Pfaff, and Burgard]{kersting07most}
Kersting, K., Plagemann, C., Pfaff, P., and Burgard, W.
\newblock Most likely heteroscedastic gaussian process regression.
\newblock In \emph{ACM International Conference Proceeding Series}, volume 227, 2007.

\bibitem[Kim et~al.(2022)Kim, Yang, and Jun]{kim22improved}
Kim, Y., Yang, I., and Jun, K.-S.
\newblock {Improved regret analysis for variance-adaptive linear bandits and horizon-free linear mixture mdps}.
\newblock In \emph{Advances in Neural Information Processing Systems (NeurIPS)}, 2022.

\bibitem[Kirschner and Krause(2018)]{kirschner18information}
Kirschner, J. and Krause, A.
\newblock Information directed sampling and bandits with heteroscedastic noise.
\newblock In \emph{Proceedings of the Conference on Learning Theory (COLT)}, pages 358--384. PMLR, 2018.

\bibitem[Lattimore and Szepesv{\'{a}}ri(2018)]{lattimore18bandit}
Lattimore, T. and Szepesv{\'{a}}ri, C.
\newblock {Bandit Algorithms}.
\newblock 2018.
\newblock URL \url{http://downloads.tor-lattimore.com/book.pdf}.

\bibitem[Neu()]{neututorial}
Neu, G.
\newblock URL \url{https://cs.bme.hu/~gergo/files/tutorial.pdf}.

\bibitem[Orabona and Jun(2024)]{orabona24tight}
Orabona, F. and Jun, K.-S.
\newblock Tight concentrations and confidence sequences from the regret of universal portfolio.
\newblock \emph{IEEE Transactions on Information Theory}, 70\penalty0 (1):\penalty0 436--455, 2024.
\newblock \doi{10.1109/TIT.2023.3330187}.

\bibitem[Pacchiano(2025)]{pacchiano25second}
Pacchiano, A.
\newblock Second order bounds for contextual bandits with function approximation.
\newblock In \emph{Proceedings of the International Conference on Learning Representations (ICLR)}, 2025.

\bibitem[Russo and {Van Roy}(2013)]{russo13eluder}
Russo, D. and {Van Roy}, B.
\newblock {Eluder dimension and the sample complexity of optimistic exploration}.
\newblock In \emph{Advances in Neural Information Processing Systems (NeurIPS)}, pages 2256--2264, 2013.

\bibitem[Stoltz(2005)]{stoltz05incomplete}
Stoltz, G.
\newblock \emph{Incomplete information and internal regret in prediction of individual sequences}.
\newblock PhD thesis, Universit{\'e} Paris Sud-Paris XI, 2005.

\bibitem[Tsybakov(2004)]{tsybakov2004introduction}
Tsybakov, A.~B.
\newblock \emph{Introduction to Nonparametric Estimation}.
\newblock Springer Series in Statistics. Springer, 2004.

\bibitem[Wagenmaker et~al.(2022)Wagenmaker, Chen, Simchowitz, Du, and Jamieson]{wagenmaker22first}
Wagenmaker, A.~J., Chen, Y., Simchowitz, M., Du, S., and Jamieson, K.
\newblock First-order regret in reinforcement learning with linear function approximation: A robust estimation approach.
\newblock In \emph{Proceedings of the International Conference on Machine Learning (ICML)}, pages 22384--22429, 2022.

\bibitem[Wainwright(2019)]{wainwright19high}
Wainwright, M.~J.
\newblock \emph{High-dimensional statistics: A non-asymptotic viewpoint}, volume~48.
\newblock Cambridge university press, 2019.

\bibitem[Wang et~al.(2024{\natexlab{a}})Wang, Oertell, Agarwal, Kallus, and Sun]{wang24more}
Wang, K., Oertell, O., Agarwal, A., Kallus, N., and Sun, W.
\newblock More benefits of being distributional: Second-order bounds for reinforcement learning.
\newblock In \emph{Proceedings of the International Conference on Machine Learning (ICML)}, 2024{\natexlab{a}}.

\bibitem[Wang et~al.(2024{\natexlab{b}})Wang, Dong, Delalleau, Zeng, Shen, Egert, Zhang, Sreedhar, and Kuchaiev]{wang24helpsteer2}
Wang, Z., Dong, Y., Delalleau, O., Zeng, J., Shen, G., Egert, D., Zhang, J.~J., Sreedhar, M.~N., and Kuchaiev, O.
\newblock Helpsteer2: Open-source dataset for training top-performing reward models.
\newblock \emph{arXiv preprint arXiv:2406.08673}, 2024{\natexlab{b}}.

\bibitem[Waudby-Smith and Ramdas(2023)]{waudby23estimating}
Waudby-Smith, I. and Ramdas, A.
\newblock Estimating means of bounded random variables by betting.
\newblock \emph{Journal of the Royal Statistical Society Series B: Statistical Methodology}, 2023.

\bibitem[Weltz et~al.(2023)Weltz, Fiez, Volfovsky, Laber, Mason, Jain, et~al.]{weltz23experimental}
Weltz, J., Fiez, T., Volfovsky, A., Laber, E., Mason, B., Jain, L., et~al.
\newblock Experimental designs for heteroskedastic variance.
\newblock \emph{Advances in Neural Information Processing Systems (NeurIPS)}, 2023.

\bibitem[Zhang et~al.(2022)Zhang, Cutkosky, and Paschalidis]{zhang22optimal}
Zhang, Z., Cutkosky, A., and Paschalidis, Y.
\newblock Optimal comparator adaptive online learning with switching cost.
\newblock \emph{Advances in Neural Information Processing Systems (NeurIPS)}, pages 23936--23950, 2022.

\bibitem[Zhao et~al.(2023{\natexlab{a}})Zhao, He, Zhou, Zhang, and Gu]{zhao23variance}
Zhao, H., He, J., Zhou, D., Zhang, T., and Gu, Q.
\newblock Variance-dependent regret bounds for linear bandits and reinforcement learning: Adaptivity and computational efficiency.
\newblock In \emph{Proceedings of the Conference on Learning Theory (COLT)}, volume 195 of \emph{Proceedings of Machine Learning Research}, pages 4977--5020. PMLR, 12--15 Jul 2023{\natexlab{a}}.

\bibitem[Zhao et~al.(2023{\natexlab{b}})Zhao, Zhou, He, and Gu]{zhao23optimal}
Zhao, H., Zhou, D., He, J., and Gu, Q.
\newblock Optimal online generalized linear regression with stochastic noise and its application to heteroscedastic bandits.
\newblock In \emph{Proceedings of the International Conference on Machine Learning (ICML)}, pages 42259--42279, 2023{\natexlab{b}}.

\bibitem[Zhao et~al.(2024)Zhao, Jun, Fiez, and Jain]{zhao24adaptive}
Zhao, Y., Jun, K.-S., Fiez, T., and Jain, L.
\newblock Adaptive experimentation when you can't experiment.
\newblock In \emph{Advances in Neural Information Processing Systems (NeurIPS)}, 2024.

\bibitem[Zhou and Gu(2022)]{zhou22computationally}
Zhou, D. and Gu, Q.
\newblock Computationally efficient horizon-free reinforcement learning for linear mixture mdps.
\newblock \emph{Advances in Neural Information Processing Systems (NeurIPS)}, pages 36337--36349, 2022.

\bibitem[Zhou et~al.(2021)Zhou, Gu, and Szepesvari]{zhou21nearly}
Zhou, D., Gu, Q., and Szepesvari, C.
\newblock {Nearly minimax optimal reinforcement learning for linear mixture markov decision processes}.
\newblock In \emph{Proceedings of the Conference on Learning Theory (COLT)}, pages 4532--4576. PMLR, 2021.

\end{thebibliography}




\clearpage
\appendix
\addcontentsline{toc}{section}{Appendix} 
\part{Appendix}

\parttoc


\section{Proof of Theorem~\ref{thm:first-order-bound}}
\begin{theorem}[Restatement of Theorem~\ref{thm:first-order-bound}]
Under log loss, we define: 
\begin{align*}
  L_{\log}(f) := \sum_{(x,y) \in D_n} y \ln (\fr1{f_x}) + (1-y) \ln (\fr1{1-f_x}). 
\end{align*}
Let $\hf = \argmin_{f \in \cF} L_{\log}(f)$. 

Then, with probability at least $1-\dt$,
\begin{align*}
  \EE_x |\hf_x - f^*_x| 
  \leq 8\sqrt{\EE[f^*_x(1-f^*_x)] \fr {\ln (|\cF|/\delta)} n} + 4 \fr {\ln (|\cF|/\delta)} n~.
\end{align*}
\end{theorem}
\begin{proof}
    For any $f \in \cF$, define
    \[
    H(f) := \fr 12 (L(f^*) - L(f)) = \sum_{(x,y) \in D_n} \fr 12y \ln (\fr{f_x}{f^*_x}) + \fr 12(1-y) \ln (\fr{1-f_x}{1-f^*_x}) ~.
    \]
    
       Inspired by~\citet{foster21efficient}, for a fixed $f \in \cF$, consider the martingale of
      \begin{align*}
        \fr{\exp(H(f))}{\EE[\exp(H(f))]}
      \end{align*}
      and apply Markov's inequality to obtain that 
      \begin{align*}
        1-\dt \le \PP\del{\fr1n H(f) \le \fr1n \ln(\EE[ \exp(H(f) )]) + \fr {\ln (1/\delta)} n} ~. 
      \end{align*}
      Taking a union bound over $f \in \cF$, 
      \begin{align*}
        1-\dt \le \PP\del{\forall f \in \cF, ~ \fr1n H(f) \le \fr1n \ln(\EE[ \exp(H(f) )]) + \fr {\ln (|\cF|/\delta)} n} ~. 
      \end{align*}
      The rest of the proof conditions on the event that \[
      \forall f \in \cF, ~ \fr1n H(f) \le \fr1n \ln(\EE[ \exp(H(f) )]) + \fr {\ln (|\cF|/\delta)} n.
      \]
      By the definition of $\hf$, $H(\hf) \geq 0$, which implies that 
      \begin{align}
          0 \le \fr1n \ln(\EE[ \exp(H(\hf) )]) + \fr {\ln (|\cF|/\delta)} n.
          \label{eqn:first-order-1}
      \end{align}
      Next, we upper bound $\fr1n \ln(\EE[ \exp(H(\hf) )])$. 
      \begin{align*}
          \fr1n \ln(\EE[ \exp(H(\hf) )]) 
          &= \fr1n \ln\del{\EE[ \prod_{(x,y) \in D_n} (\fr{\hf_x}{f^*_x})^{\fr 12 y} (\fr{1-\hf_x}{1-f^*_x})^{\fr 12(1-y)} ]}
          \\&= \ln\del{\EE[ (\fr{\hf_x}{f^*_x})^{\fr 12 y} (\fr{1-\hf_x}{1-f^*_x})^{\fr 12(1-y)} ]}
          \tag{independence}
      \end{align*}
      \begin{align*}
          \EE[ (\fr{\hf_x}{f^*_x})^{\fr 12 y} (\fr{1-\hf_x}{1-f^*_x})^{\fr 12(1-y)} ] 
          &= \EE \exp \del{ \fr 12y \ln (\fr{\hf_x}{f^*_x}) + \fr 12(1-y) \ln (\fr{1-\hf_x}{1-f^*_x}) }
          \\&= \EE \exp \del{ \EE' \fr 12y' \ln (\fr{\hf_x}{f^*_x}) + \fr 12(1-y') \ln (\fr{1-\hf_x}{1-f^*_x}) }
          \tag{$y'\sim \Bernoulli(y) $ }
          \\&\leq \EE \EE'\exp \del{\fr 12y' \ln (\fr{\hf_x}{f^*_x}) + \fr 12(1-y') \ln (\fr{1-\hf_x}{1-f^*_x}) }
          \tag{Jensen's inequality} 
          \\&= \EE_x f^*_x\cd \sqrt{\fr{\hf_x}{f^*_x} } + (1-f^*_x) \cd \sqrt{\fr{1-\hf_x}{1-f^*_x} }
          \\&= \EE_x [f^*_x \hf_x  + (1-f^*_x) (1-\hf_x)] 
      \end{align*}
      Combining with Eqn.~\eqref{eqn:first-order-1}, 
      \begin{align}
           \fr {\ln (|\cF|/\delta)} n
           &\geq -\ln (\EE_x [f^*_x \hf_x  + (1-f^*_x) (1-\hf_x)] )\nonumber
           \\&= -\ln (1 - \EE [1 - f^*_x \hf_x  - (1-f^*_x) (1-\hf_x)])
           \nonumber
        \\&\ge \EE [1 - f^*_x \hf_x  - (1-f^*_x) (1-\hf_x)] \nonumber
        \tag{$\ln(1+x) \leq x$} \nonumber
        \\&=   \EE[ \fr12 (\sqrt{f^*_x} - \sqrt{\hf_x})^2 + \fr12 (\sqrt{1-f^*_x} - \sqrt{1-\hf_x})^{2} ] \nonumber
        \\&= \EE[D^2(f^*_x, \hf_x)]
        \label{eqn:first-order-2}
      \end{align}
    where $D^2 (p,q)$ for scalers $p,q \in [0,1]$ denotes the Hellinger distance between two Bernoulli distributions with parameters $p$ and $q$: i.e., $D^2 (p,q) 
      = \fr12 (\sqrt{p}-\sqrt{q})^2 + \fr12 (\sqrt{1-p} - \sqrt{1-q} )^2$. 
      
      From the proof of Proposition 3 of~\citet{foster21efficient}, we know that
    \begin{align*}
      D^2 (p,q) 
      &= \fr12 (\sqrt{p}-\sqrt{q})^2 + \fr12 (\sqrt{1-p} - \sqrt{1-q} )^2
    \\&= \fr{(p-q)^2}{2} \cd \del{ \fr{1}{(\sqrt{p} + \sqrt{q})^2} + \fr{1}{(\sqrt{1-p} + \sqrt{1-q} )^2}  } 
    \\&\ge \fr{(p-q)^2}{2} \cd \del{ \fr{1}{(\sqrt{p} + \sqrt{q})^2 \wed (\sqrt{1-p} + \sqrt{1-q} )^2 }   } 
    \\&\ge \fr{(p-q)^2}{4} \cd \del{ \fr{1}{(p + q) \wed (1-p + 1-q) }   }  \tag{$(a + b)^2 \le 2a^2 + 2b^2$ }
    \end{align*}
    Let $g(p,q) = (p + q) \wed (1-p + 1-q)$.
    Then, by Eqn.~\eqref{eqn:first-order-2}, 
    \begin{align*}
      \EE\sbr[2]{ (f^*_x - \hf_x)^2 \cd \fr{1}{2 g(f^*_x, \hf_x)}} \le 2\fr {\ln (|\cF|/\delta)} n
    \end{align*}
    Using $\fr{A^2}{2B} = \max_{\eta > 0} \eta A - \fr{\eta^2}{2}  B$ for $A,B>0$, we have, for any $\eta>0$, 
    \begin{align*}
      2\fr {\ln (|\cF|/\delta)} n
      &\ge \EE[\max_\eta \eta |f^*_x - \hf_x| - \fr{\eta^2}{2} g(f^*_x, \hf_x)]
    \\&\ge \max_\eta \eta \EE [|f^*_x - \hf_x|]  - \fr{\eta^2}{2} \EE[g(f^*_x, \hf_x)] \tag{Jensen} 
    \\ \implies
    \EE[|f^*_x - \hf_x|] &\le \min_\eta \fr \eta 2 \EE[g(f^*_x, \hf_x)] +\fr1 \eta \fr {2 \ln (|\cF|/\delta)} n~.
    \end{align*}
    Note that 
    \begin{align*}
      \EE g(f^*_x, \hf_x)
      &= \EE[ (f^*_x + \hf_x) \wed (1-f^*_x + 1-\hf_x) ]
    \\&\le \EE[ (|f^*_x - \hf_x| + 2f^*_x) \wed (|f^*_x - \hf_x| + 2(1-f^*_x)) ]
    \\&=   \EE[ |f^*_x - \hf_x| + (2f^*_x \wed 2(1-f^*_x)) ]
    \\&=   \EE[ |f^*_x - \hf_x|] + 2\EE[ f^*_x \wed (1-f^*_x) ]
    \\&\le   \EE[ |f^*_x - \hf_x|] + 4\EE[ f^*_x(1-f^*_x) ].
    \end{align*}
    Then,
    \begin{align*}
      \EE[|f^*_x - \hf_x| ] 
      &\le \fr{\eta}{2} \EE[ |f^*_x - \hf_x|] + 2\eta \EE[ f^*_x(1-f^*_x) ]  +  \fr1\eta \fr {2\ln (|\cF|/\delta)} n 
    \\&\le \fr12 \EE[ |f^*_x - \hf_x|] + 2\eta \EE[ f^*_x(1-f^*_x) ]  +  \fr1\eta \fr {2\ln (|\cF|/\delta)} n \tag{assume $\eta \le1$}
    \\ \implies
      \EE[|f^*_x - \hf_x| ] 
      &\le 4\eta \EE[ f^*_x(1-f^*_x) ]  +  \fr4\eta \fr {\ln (|\cF|/\delta)} n
    \end{align*}
    We can choose $\eta = 1 \wed \sqrt{\fr{\ln (|\cF|/\delta)/n}{\EE[f^*_x(1-f^*_x)]}}$, which satisfies the assumption above, to arrive at
    \begin{align*}
      \EE[|f^*_x - \hf_x| ] &\leq 8\sqrt{\EE[f^*_x(1-f^*_x)] \fr {\ln (|\cF|/\delta)} n} + 4 \fr {\ln (|\cF|/\delta)} n~.
    \end{align*}
\end{proof}

\section{Proof of Theorem~\ref{thm:second-order-bound}}


\begin{definition}
    We first provide definitions for new quantities that are used throughout the proof of Theorem~\ref{thm:second-order-bound}.  
    \begin{align*}
        \Dt_x :=& f^*_x - f_x
        \\ \warDt_{h,x,\phi,c}:=&  \clip{\phi(h_x-f_x)}_{[-c, c]}
        \\ U_x :=& \max\{(-f^*_x) \fr{-\warDt_{h,x,\phi,c}}{1 + \Dt_x\warDt_{h,x,\phi,c}},~ (1-f^*_x) \fr{-\warDt_{h,x,\phi,c}}{1 + \Dt_x\warDt_{h,x,\phi,c}}\}
    \end{align*}
\end{definition}

\begin{lemma} \label{lem:basic-properties-U}
  For any $x \in \cX$, we have: 
  \begin{enumerate}
      \item $\warDt_{f^*,x,\phi,c} = \sign(f^*_x - f_x) \del{\phi |f^*_x - f_x| \wed c  }$, and $\Dt_x\warDt_{f^*,x,\phi,c} \geq 0$.
      \item $U_x \leq \fr 14$.
  \end{enumerate}
\end{lemma}
\begin{proof}
    \begin{enumerate}
        \item By the definition of $\warDt_{h,x,\phi,c}$, one can see  $\warDt_{f^*,x,\phi,c} = \clip{\phi(f^*_x-f_x)}_{[-c, c]} = \sign(f^*_x - f_x) \del{\phi |f^*_x - f_x| \wed c  }$, and $\Dt_x\warDt_{f^*,x,\phi,c} \geq 0$ for all $x$.
\item 
Note that  
\begin{align*}
    |\warDt_{f^*,x,\phi,c}|
    =& \phi |f^*_x - f_x| \wed c
    \\ \leq& c
    \\ \leq& \fr 14
    \tag{$c \leq \fr 14$}
\end{align*}

If $\warDt_{f^*,x,\phi,c} \geq 0$, 
\begin{align*}
    U_x =& \max\{(-f^*_x) \fr{-\warDt_{f^*,x,\phi,c}}{1 + \Dt_x\warDt_{f^*,x,\phi,c}},~ (1-f^*_x) \fr{-\warDt_{f^*,x,\phi,c}}{1 + \Dt_x\warDt_{f^*,x,\phi,c}}\}
    \\=& f^*_x \fr{\warDt_{f^*,x,\phi,c}}{1 + \Dt_x\warDt_{f^*,x,\phi,c}}
    \\ \leq& \warDt_{f^*,x,\phi,c}
    \tag{$\forall x, ~\Dt_x\warDt_{f^*,x,\phi,c} \geq 0, ~0 \leq f^*_x \leq 1$ }
    \\ \leq& \fr 14
    \tag{$|\warDt_{f^*,x,\phi,c}| \leq \fr 14$}
\end{align*}
Similarly, we can show that if $\warDt_{f^*,x,\phi,c} \leq 0$, then $U_x = (1-f^*_x) \fr{-\warDt_{f^*,x,\phi,c}}{1 + \Dt_x\warDt_{f^*,x,\phi,c}} \leq \fr 14$. 
    \end{enumerate}
\end{proof}

\begin{lemma}\label{lem:log_1mx} 
  Let $a \in (0,1)$.
  Then,
  $\forall x \in [0,a],~ \ln(1-x) \ge \fr{-\ln(1-a)}{a}\cd (-x).$ 
\end{lemma}
\begin{proof}
    Given the concavity of $\ln(1-x)$, for any $x \in [0,a]$, the function lies above the secant line connecting $(0, \ln(1-0))$ and $(a, \ln(1-a))$. 

    The equation of the secant line is: \[
    y = \fr{\ln(1-a)}{a} x. 
    \]

    By concavity: \[
    \ln(1-x) \geq \fr{\ln(1-a)}{a} x. 
    \]
\end{proof}

\begin{lemma} \label{lem:positive-loss-deviation-with-union-bound}
  Let $\delta \in (0, \fr{1}{|\cF|})$. We have, 
  \begin{align*}
    1-|\cF|\dt\le\PP\del{\forall h\in \cF, ~\phi \in [0,\war\phi], ~c \in [0,\fr 14], ~\fr1n H_{\phi,c}(h,f^*)
      \le \fr1n\ln(8\war\phi n^2/\dt)} 
  \end{align*}
\end{lemma}
\begin{proof}
  The plan is to fix $h$ and show 
  \begin{align*}
    1-\dt \le \PP\del[2]{\forall \phi^* \in [0,\war\phi], c^* \in [0,\fr14],~\fr1n H_{\phi^*, c^*}(h , f^*) \le \fr1n\ln(8\war\phi n^2/\dt)} 
  \end{align*}
  and then take the union bound over $h\in \cF$.

      Let $\eps>0$ be a small number to be chosen later. Discretize $[0,\war\phi]\times [0,\fr 14]$ as blocks of length $\eps$ by $\eps$. The number of such blocks is $\fr{\war\phi}{4\eps^2}$. 
  For any $(\phi^*, c^*)$, there is block, such that $(\phi^*, c^*)$ belongs to this block. Let $U'$ be the uniform distribution supported on this block.

  We start from the martingale
  \begin{align*}
    \fr{\EE_{(\phi,c)\sim U'}[\exp(H_{\phi,c}(h, f^*))]}{\EE_{(\phi,c) \sim U',\{x,y\}\sim D^n}[\exp(H_{\phi,c}(h, f^*))]} 
  \end{align*}
  Using Markov's inequality, we have, w.p. at least $1-\dt/(\fr{\war\phi}{4\eps^2})$, 
  \begin{align} 
    \ln(\EE_{(\phi,c)\sim U'}[\exp(H_{\phi,c}(h, f^*))]) 
    &\le \ln(\EE_{(\phi,c) \sim U',\{(x,y)\}\sim D^n}[\exp(H_{\phi,c}(h, f^*))]) + \ln(\fr{\war\phi}{4\eps^2 \dt }) 
    \nonumber
    \\&= \ln(\EE_{(\phi,c) \sim U'} (\EE_{\{(x,y)\}\sim D}[1+ (y-f^*)\clip{\phi(h_x-f_x)}_{[-c, c]}] )^n ) + \ln(\fr{\war\phi}{4\eps^2 \dt })
    \tag{independence}
    \nonumber
  \\&= \ln(\fr{\war\phi}{4\eps^2 \dt }) \label{eqn:positive-loss-deviation}
  \end{align}
  Taking a union bound over all $\fr{\war\phi}{4\eps^2}$ blocks, we have with probability at least $1- \delta$, for any $U$ that is a uniform distribution on any block, 
    \begin{align*}
    \ln(\EE_{(\phi,c)\sim U}[\exp(H_{\phi,c}(h, f^*))]) 
&\le \ln(\fr{\war\phi}{4\eps^2 \dt })
  \end{align*}
  We desire to lower bound the LHS of Equation~\eqref{eqn:positive-loss-deviation} above as $H_{\phi^*, c^*}(h , f^*)$ plus some extra terms for any $(\phi^*, c^*)$ that belongs to the support of $U'$.
 
  Note that 
  \begin{align*}
    \EE_{(\phi,c)\sim U'}[\exp(H_{\phi,c}(h,f^*))]
  &=  \EE_{(\phi,c)\sim U'}\del{\prod_{(x,y)} \del{1 + (y-f^*_x) \clip{\phi(h_x - f^*_x)}_{[-c,c]}} }~.
  \end{align*}
  Note that if $|\phi^*-\phi|\le \eps$ and $|c^*-c|\le \eps$, then using 1-Lipschitzness of $F_1(\phi) = 1 + (y-f^*_x) \clip{\phi(h_x - f^*_x)}_{[-c,c]}$, and 1-Lipschitzness of $F_2(c) = 1 + (y-f^*_x) \clip{\phi(h_x - f^*_x)}_{[-c,c]}$, 
  \begin{align*}
    &1 + (y-f^*_x) \clip{\phi(h_x - f^*_x)}_{[-c,c]}
  \\&\ge 1 + (y-f^*_x) \clip{\phi^*(h_x - f^*_x)}_{[-c,c]} - \eps
  \\&\ge 1 + (y-f^*_x) \clip{\phi^*(h_x - f^*_x)}_{[-c^*,c^*]} - 2\eps
  \\&=  (1 + (y-f^*_x) \clip{\phi^*(h_x - f^*_x)}_{[-c^*,c^*]})\cd (1 - \fr{2\eps}{1 + (y-f^*_x) \clip{\phi^*(h_x - f^*_x)}_{[-c^*,c^*]}} )
  \\&\ge  (1 + (y-f^*_x) \clip{\phi^*(h_x - f^*_x)}_{[-c^*,c^*]})\cd (1 - \fr83 {\eps} ) \tag{$c^*\le\fr14$ }
  \end{align*}
  Thus,
  \begin{align*}
    \ln(\EE_{(\phi,c)\sim U}[\exp(H_{\phi,c}(h,f^*))])
    &\ge \sum_{(x,y)}\ln\del{1 + (y-f^*_x) \clip{\phi^*(h_x - f^*_x)}_{[-c^*,c^*]}} + n\ln(1-\fr83 \eps)
    \\&\ge \sum_{(x,y)}\ln\del{1 + (y-f^*_x) \clip{\phi^*(h_x - f^*_x)}_{[-c^*,c^*]}} - n\eps \tag{Lemma~\ref{lem:log_1mx}; $\eps\le \fr18$}
  \end{align*}
  This implies that
  \begin{align*}
    \fr1n H_{\phi^*,c^*}(h,f^*)
    &\le \eps + \fr1n\ln(\fr{\war\phi}{4\eps^2 \dt })
  \end{align*}
  Choosing $\eps = \fr1{4n}$, the RHS of above inequality can be upper bounded as: 
  \begin{align*}
      \eps + \fr1n\ln(\fr{\war\phi}{4\eps^2 \dt })
      \le \fr1n\ln(8\war\phi n^2/\dt)
  \end{align*}
  concluding the proof. 
\end{proof}

\begin{lemma}
    \label{lem:negative-loss-deviation-with-union-bound}
  Let $\delta \in (0, \fr{1}{|\cF|})$. 
  Then, 
  \begin{align*}
    1-|\cF|\dt\le\PP\bigg(&\forall f\in \cF, ~\phi \in [0,\war\phi], ~c \in [0,\fr 14], \\&-\fr1n H_{\phi,c}(f^*,f)
      \le \EE_x\sbr[2]{ -\fr{\Dt_x\warDt_{f^*,x,\phi,c}}{1 + \Dt_x\warDt_{f^*,x,\phi,c}}  + \fr 43\cd\sig_x^2 \warDt^2_{f^*,x,\phi,c} } + \fr1n\ln(24\war\phi n^2/\dt) \bigg)
  \end{align*}
\end{lemma} 
\begin{proof}
  The plan is to fix $f$ and show 
  \begin{align*}
    1-\dt \le \PP\bigg(&\forall \phi^* \in [0,\war\phi], c^* \in [0,\fr14],~-\fr1n H_{\phi^*,c^*}(f^*,f) 
    \\&\le \EE_x\sbr[2]{ -\fr{\Dt_x\warDt_{f^*,x,\phi^*,c^*}}{1 + \Dt_x\warDt_{f^*,x,\phi^*,c^*}}  + \fr 43\cd\sig_x^2 \warDt^2_{f^*,x,\phi^*,c^*} } + \fr1n\ln(24\war\phi n^2/\dt) \bigg)
  \end{align*}
  and then take the union bound over $f\in \cF$.

  Let $\eps>0$ be a small number to be chosen later. Discretize $[0,\war\phi]\times [0,\fr 14]$ as blocks of length $\eps$ by $\eps$. The number of such blocks is $\fr{\war\phi}{4\eps^2}$. 
  For any $(\phi^*, c^*)$, there is block, such that $(\phi^*, c^*)$ belongs to this block. Let $U'$ be the uniform distribution supported on this block. 
  
  We start from the martingale
  \begin{align*}
    \fr{\EE_{(\phi,c)\sim U'}[\exp(-H_{\phi,c}(f^*,f))]}{\EE_{(\phi,c) \sim U',\{x,y\}\sim D}[\exp(-H_{\phi,c}(f^*,f))]} 
  \end{align*}
  Using Markov's inequality, we have, w.p. at least $1-\dt/(\fr{\war\phi}{4\eps^2})$, 
  \begin{align} \label{eqn:loss-deviation}
    &\ln(\EE_{(\phi,c)\sim U'}[\exp(-H_{\phi,c}(f^*,f))]) \nonumber
    \\ \le& \ln(\EE_{(\phi,c) \sim U',\{(x,y)\}\sim D^n}[\exp(-H_{\phi,c}(f^*,f))]) + \ln(\fr{\war\phi}{4\eps^2}/\dt) \nonumber
    \\ \le& n \EE_{(\phi,c)\sim U'}\EE_x\sbr[2]{ -\fr{\Dt_x\warDt_{f^*,x,\phi,c}}{1 + \Dt_x\warDt_{f^*,x,\phi,c}}  + \fr{1}{(1+\Dt_x \warDt_{f^*,x,\phi,c})^3(1-U_x)}\cd\sig_x^2 \warDt^2_{f^*,x,\phi,c} } + \ln(\fr{\war\phi}{4\eps^2}/\dt)
  \end{align}
where the last inequality is by Lemma~\ref{lem:concentration-without-union-bound}.

  Taking a union bound over all $\fr{\war\phi}{4\eps^2}$ blocks, we have with probability at least $1- \delta$, for any $U$ that is a uniform distribution on any block, 
    \begin{align*}
    &\ln(\EE_{(\phi,c)\sim U}[\exp(-H_{\phi,c}(f^*,f))])
\\ \le& n \EE_{(\phi,c)\sim U'}\EE_x\sbr[2]{ -\fr{\Dt_x\warDt_{f^*,x,\phi,c}}{1 + \Dt_x\warDt_{f^*,x,\phi,c}}  + \fr{1}{(1+\Dt_x \warDt_{f^*,x,\phi,c})^3(1-U_x)}\cd\sig_x^2 \warDt^2_{f^*,x,\phi,c} } + \ln(\fr{\war\phi}{4\eps^2}/\dt)
  \end{align*}

We upper bound the RHS of~\eqref{eqn:loss-deviation} as follows: 

\begin{align*}
    & \EE_{(\phi,c)\sim U'}\EE_x\sbr[2]{ -\fr{\Dt_x\warDt_{f^*,x,\phi,c}}{1 + \Dt_x\warDt_{f^*,x,\phi,c}}  + \fr{1}{(1+\Dt_x \warDt_{f^*,x,\phi,c})^3(1-U_x)}\cd\sig_x^2 \warDt^2_{f^*,x,\phi,c} }
    \\\le& \EE_{(\phi,c)\sim U'}\EE_x\sbr[2]{ -\fr{\Dt_x\warDt_{f^*,x,\phi,c}}{1 + \Dt_x\warDt_{f^*,x,\phi,c}}  + \fr 43\cd\sig_x^2 \warDt^2_{f^*,x,\phi,c} }
    \tag{By Lemma~\ref{lem:basic-properties-U}: $\Dt_x\warDt_{g,x,c} \geq 0$, ~$U_x \leq \fr 14$}
\end{align*}

One can see that $\warDt_{f^*,x,\phi,c} = \clip{\phi(f^*_x - f_x)}_{[-c,c]}$ is 1-Lipschitz in $\phi$ and 1-Lipschitz in $c$, i.e., $F_1(\phi) =\warDt_{f^*,x,\phi,c}$ is 1-Lipschitz, and $F_2(c) = \warDt_{f^*,x,\phi,c}$ is 1-Lipschitz. 
Further, $F_1^2(\phi)$ and $F_2^2(c)$ are 1-Lipschitz since $\warDt_{f^*,x,\phi,c} \leq c \leq \fr 14$. In addition, since for $x \in [0, \fr 14]$, $|\fr{\dif}{\dif x} \fr{x}{1+x}| = \fr{1}{(1+x)^2} \leq 1$, $\fr{\Dt_x\warDt_{f^*,x,\phi,c}}{1 + \Dt_x\warDt_{f^*,x,\phi,c}}$ is 1-Lipschitz in $\phi$ and 1-Lipschitz in $c$. 

  Note that if $|\phi^*-\phi|\le \eps$ and $|c^*-c|\le \eps$, then 
  using Lipchitzness arguments above, as well as $\sig_x^2 \le \fr 14, ~\forall x$, we have
  \begin{align*}
       -\fr{\Dt_x\warDt_{f^*,x,\phi,c}}{1 + \Dt_x\warDt_{f^*,x,\phi,c}}  + \fr 43\cd\sig_x^2 \warDt^2_{f^*,x,\phi,c} 
      \le& -\fr{\Dt_x\warDt_{f^*,x,\phi^*,c}}{1 + \Dt_x\warDt_{f^*,x,\phi^*,c}}  + \fr 43\cd\sig_x^2 \warDt^2_{f^*,x,\phi^*,c}  + 2\eps
      \\ \le&  -\fr{\Dt_x\warDt_{f^*,x,\phi^*,c^*}}{1 + \Dt_x\warDt_{f^*,x,\phi^*,c^*}}  + \fr 43\cd\sig_x^2 \warDt^2_{f^*,x,\phi^*,c^*}  + 4\eps
  \end{align*}
  This implies that, 
  \begin{align*}
      &\EE_{(\phi,c)\sim U'}\EE_x\sbr[2]{ -\fr{\Dt_x\warDt_{f^*,x,\phi,c}}{1 + \Dt_x\warDt_{f^*,x,\phi,c}}  + \fr 43\cd\sig_x^2 \warDt^2_{f^*,x,\phi,c}  }
      \\ \le& \EE_x\sbr[2]{  -\fr{\Dt_x\warDt_{f^*,x,\phi^*,c^*}}{1 + \Dt_x\warDt_{f^*,x,\phi^*,c^*}}  + \fr 43\cd\sig_x^2 \warDt^2_{f^*,x,\phi^*,c^*} } + 4\eps
  \end{align*}
  For the LHS of~\eqref{eqn:loss-deviation}, 
    \begin{align*}
    \EE_{(\phi,c)\sim U'}[\exp(-H_{\phi,c}(f^*,f))]
  &=  \EE_{(\phi,c)\sim U'}\del{\prod_{(x,y)} \fr{1}{ \del{1 + (y-f_x) \clip{\phi(f^*_x - f_x)}_{[-c,c]}} } }~.
  \end{align*}
  Note that if $|\phi^*-\phi|\le \eps$ and $|c^*-c|\le \eps$, then using 1-Lipschitzness of $F_3(\phi) = 1 + (y-f_x) \clip{\phi(f^*_x - f_x)}_{[-c,c]}$, and 1-Lipschitzness of $F_4(c) = 1 + (y-f_x) \clip{\phi(f^*_x - f_x)}_{[-c,c]}$, 
  \begin{align*}
    &1 + (y-f_x) \clip{\phi(f^*_x - f_x)}_{[-c,c]}
  \\&\le 1 + (y-f_x) \clip{\phi^*(f^*_x - f_x)}_{[-c,c]} + \eps
  \\&\le 1 + (y-f_x) \clip{\phi^*(f^*_x - f_x)}_{[-c^*,c^*]} + 2\eps
  \\&=  (1 + (y-f_x) \clip{\phi^*(f^*_x - f_x)}_{[-c^*,c^*]})\cd (1 + \fr{2\eps}{1 + (y-f_x) \clip{\phi^*(f^*_x - f_x)}_{[-c^*,c^*]}} )
  \\&\le  (1 + (y-f_x) \clip{\phi^*(f^*_x - f_x)}_{[-c^*,c^*]})\cd (1 + \fr83 {\eps} ) \tag{$c^*\le\fr14$ }
  \end{align*}
  Thus,
  \begin{align*}
    \ln\del{\EE_{(\phi,c)\sim U'}[\exp(-H_{\phi,c}(f^*,f))]}
    &=  \EE_{(\phi,c)\sim U'}\del{\prod_{(x,y)} \fr{1}{ \del{1 + (y-f_x) \clip{\phi(f^*_x - f_x)}_{[-c,c]}} } }
    \\&\ge \sum_{(x,y)}-\ln\del{1 + (y-f_x) \clip{\phi^*(f^*_x - f_x)}_{[-c^*,c^*]}} - n\ln(1+\fr83 \eps)
    \\&\ge  \sum_{(x,y)}-\ln\del{1 + (y-f_x) \clip{\phi^*(f^*_x - f_x)}_{[-c^*,c^*]}} - n\fr83\eps \tag{$\ln(1+x)\le x$}
  \end{align*}
  Combining the bounds for the LHS and RHS of~\eqref{eqn:loss-deviation}, 
  \begin{align*}
      &\sum_{(x,y)}-\ln\del{1 + (y-f_x) \clip{\phi^*(f^*_x - f_x)}_{[-c^*,c^*]}} - n\fr83\eps 
    \\&\leq
      n \EE_x\sbr[2]{  -\fr{\Dt_x\warDt_{f^*,x,\phi^*,c^*}}{1 + \Dt_x\warDt_{f^*,x,\phi^*,c^*}}  + \fr 43\cd\sig_x^2 \warDt^2_{f^*,x,\phi^*,c^*} }  + 4n\eps + \ln(\fr{\war\phi}{4\eps^2}/\dt)
  \end{align*}
    This implies that
  \begin{align*}
    -\fr1n H_{\phi^*,c^*}(f^*,f)
    &\le \fr{20}3\eps + \EE_x\sbr[2]{  -\fr{\Dt_x\warDt_{f^*,x,\phi^*,c^*}}{1 + \Dt_x\warDt_{f^*,x,\phi^*,c^*}}  + \fr 43\cd\sig_x^2 \warDt^2_{f^*,x,\phi^*,c^*} } + \fr1n\ln(\fr{\war\phi}{4\eps^2}/\dt)
  \end{align*}
  Choosing $\eps = \fr1{4n}$, 
  \begin{align*}
      -\fr1n H_{\phi^*,c^*}(f^*,f)
    &\le \fr{20}3\eps + \EE_x\sbr[2]{  -\fr{\Dt_x\warDt_{f^*,x,\phi^*,c^*}}{1 + \Dt_x\warDt_{f^*,x,\phi^*,c^*}}  + \fr 43\cd\sig_x^2 \warDt^2_{f^*,x,\phi^*,c^*} } + \fr1n\ln(\fr{\war\phi}{4\eps^2}/\dt)
      \\& \le \EE_x\sbr[2]{  -\fr{\Dt_x\warDt_{f^*,x,\phi^*,c^*}}{1 + \Dt_x\warDt_{f^*,x,\phi^*,c^*}}  + \fr 43\cd\sig_x^2 \warDt^2_{f^*,x,\phi^*,c^*} } + \fr1n\ln(24\war\phi n^2/\dt) 
  \end{align*}
\end{proof}

\begin{lemma}
  Recall the definition of the loss function $H_{\phi,c}$ and $U_x = \max\{(-f^*_x) \fr{-\warDt_{h,x,\phi,c}}{1 + \Dt_x\warDt_{h,x,\phi,c}},~ (1-f^*_x) \fr{-\warDt_{h,x,\phi,c}}{1 + \Dt_x\warDt_{h,x,\phi,c}}\}$.
  Let $V$ be a distribution of $(\phi,c)$ supported on a subset of $[0,\war\phi]\times [0,\fr 14]$.
  Then for any $h,f\in\cF$, we have
  \begin{align*}
      & \ln(\EE_{(\phi,c) \sim V,\{(x,y)\}\sim D^n}[\exp(- H_{\phi,c}(h, f))])
    \\ \le& n \EE_{(\phi,c)\sim V}\EE_x\sbr[2]{ -\fr{\Dt_x\warDt_{h,x,\phi,c}}{1 + \Dt_x\warDt_{h,x,\phi,c}}  + \fr{1}{(1+\Dt_x \warDt_{h,x,\phi,c})^3(1-U_x)}\cd\sig_x^2 \warDt^2_{h,x,\phi,c} }
  \end{align*}
  \label{lem:concentration-without-union-bound}
\end{lemma}
\begin{proof}
      Let $\eta := y - f^*$, then $\forall x \in \cX, ~\EE [\eta \mid x] = 0$ and $\EE [\eta^2 \mid x] = \sig_x^2$. 
We have
\begin{align*}
  &(\EE_{(\phi,c) \sim V,\{(x,y)\}\sim D^n}[\exp(-H_{\phi,c}(h, f) )])^{\fr 1n}
\\&= \EE_{(\phi,c) \sim V,\{(x,y)\}\sim D}[(\fr{1}{1 + (y-f_x)\clip{\phi(h_x-f_x)}_{[-c, c]}})]
\\&= \EE_{(\phi,c) \sim V,\{(x,y)\}\sim D}[(\fr{1}{1 + (f^*_x+\eta-f_x)\clip{\phi(h_x-f_x)}_{[-c, c]} })]
\\&= \EE_{(\phi,c) \sim V,\{(x,y)\}\sim D}[(\fr{1}{1 + \Dt_x \warDt_{h,x,\phi,c} + \eta \warDt_{h,x,\phi,c}})]
\\&= \EE_{(\phi,c) \sim V }\EE_x \sbr{\fr{1}{1 + \Dt_x \warDt_{h,x,\phi,c}} \EE_\eta[\fr{1}{1 + \eta \warDt_{h,x,\phi,c} \cd (1 + \Dt_x \warDt_{h,x,\phi,c})^{-1}} ] }
\end{align*}
Using the fact that $\fr{1}{1+x} = 1 - x + \fr{x^2}{1+x}$ with $x = \eta \warDt_{h,x,\phi,c} \cd (1 + \Dt_x \warDt_{h,x,\phi,c})^{-1}$, we have
\begin{align*}
\EE_\eta[\fr{1}{1 + \eta \warDt_{h,x,\phi,c} \cd (1 + \Dt_x \warDt_{h,x,\phi,c})^{-1}} ]
&= 1 + \EE_\eta[\fr{\eta^2 \warDt_{h,x,\phi,c}^2}{(1 + \Dt_x \warDt_{h,x,\phi,c})^2} \cd \fr{1}{1 + \eta \warDt_{h,x,\phi,c} \cd (1 + \Dt_x \warDt_{h,x,\phi,c})^{-1}  }]  
\end{align*}
If $\warDt_{h,x,\phi,c} \geq 0$, then the RHS $\leq 1+ \fr{\sig_x^2 \warDt_{h,x,\phi,c}^2}{(1 + \Dt_x \warDt_{h,x,\phi,c})^2} \cd \fr{1}{1 + (-f^*) \warDt_{h,x,\phi,c} \cd (1 + \Dt_x \warDt_{h,x,\phi,c})^{-1} }  $. Else if $\warDt_{h,x,\phi,c} < 0$, then the RHS $\leq 1+ \fr{\sig_x^2 \warDt_{h,x,\phi,c}^2}{(1 + \Dt_x \warDt_{h,x,\phi,c})^2} \cd \fr{1}{1 + (1-f^*) \warDt_{h,x,\phi,c} \cd (1 + \Dt_x \warDt_{h,x,\phi,c})^{-1} }  $.  

  Thus, with $U_x = \max\{(-f^*_x) \fr{-\warDt_{h,x,\phi,c}}{1 + \Dt_x\warDt_{h,x,\phi,c}},~ (1-f^*_x) \fr{-\warDt_{h,x,\phi,c}}{1 + \Dt_x\warDt_{h,x,\phi,c}}\}$, 
  \begin{align*}
    &\fr1n\ln(\EE_{(\phi,c) \sim V,\{(x,y)\}\sim D^n}[\exp(-H_{\phi,c}(h, f))])
    \\\leq& \ln \EE_{(\phi,c) \sim V } \EE_x \sbr{\fr{1}{1 + \Dt_x\warDt_{h,x,\phi,c}} \del{ 1+\sig_x^2 \warDt_{h,x,\phi,c}^2 \cd \fr{1}{(1 + \Dt_x\warDt_{h,x,\phi,c})^2 (1-U_x)} } }
    \\\le& \EE_{(\phi,c) \sim V } \EE_x \sbr{\fr{1}{1 + \Dt_x\warDt_{h,x,\phi,c}} \del{ 1+\sig_x^2 \warDt_{h,x,\phi,c}^2 \cd \fr{1}{(1 + \Dt_x\warDt_{h,x,\phi,c})^2 (1-U_x)} } -1} 
    \tag{$\ln x \leq x-1$}
    \\=& \EE_{(\phi,c) \sim V } \EE_x \sbr{\fr{1}{1 + \Dt_x\warDt_{h,x,\phi,c}} \del{\sig_x^2 \warDt_{h,x,\phi,c}^2 \cd \fr{1}{(1 + \Dt_x\warDt_{h,x,\phi,c})^2 (1-U_x)} - \Dt_x\warDt_{h,x,\phi,c} } } 
  \end{align*}
  completing the proof. 
\end{proof}

\begin{theorem}[Restatement of Theorem~\ref{thm:second-order-bound}]
Recall that
\begin{align*}
  L(f) := \max_{h \in \cF} \max_{\phi \in [0,\war\phi]} \max_{c \in [0,\fr14]} \fr 1n \sum_{(x,y) \in D_n} \ln \del{ 1+ (y-f_x) \clip{\phi(h_x-f_x)}_{[-c, c]} } 
\end{align*}
With probability at least $1-\dt$, $\forall f \in \cF$,
\begin{align*}
  &\EE_x |f_x - f^*_x| 
  \\&\le \sqrt{  \fr{25}{12}  \EE \sig_{x}^2 \cd \del{\fr2n\ln\del{\fr{48 |\cF| \war\phi n^2}{\dt}} + (L(f) - L(f^*) )}  } +   \fr6n\ln\del{\fr{48 |\cF| \war\phi n^2}{\dt}} + \fr 52(L(f) - L(f^*) )
\end{align*}
    \label{thm:second-order-main-theorem}
\end{theorem}
\begin{proof}
   Define the events 
\begin{align*}
    A_1 &:= \cbr[3]{\forall h\in \cF, ~\phi \in [0,\war\phi], ~c \in [0,\fr 14], ~\fr1n H_{\phi,c}(h,f^*)
      \le \fr1n\del{\fr{16 |\cF| \war\phi n^2}{\dt}}   } \\
    A_2 &:= \bigg\{ \forall f\in \cF, ~\phi \in [0,\war\phi], ~c \in [0,\fr 14], ~
    \\ &\qquad-\fr1n    H_{\phi,c}(f^*,f)
      \le \EE_x\sbr[2]{ -\fr{\Dt_x\warDt_{f^*,x,\phi,c}}{1 + \Dt_x\warDt_{f^*,x,\phi,c}}  + \fr 43\cd\sig_x^2 \warDt^2_{f^*,x,\phi,c} } + \fr1n\ln\del{\fr{48 |\cF| \war\phi n^2}{\dt}} \bigg\} \\
    A &:= A_1 \cap A_2
\end{align*}

    By Lemma~\ref{lem:positive-loss-deviation-with-union-bound}, $\PP(A_1) \geq 1-\fr{\dt}{2}$; by Lemma~\ref{lem:negative-loss-deviation-with-union-bound}, $\PP(A_2) \geq 1-\fr{\dt}{2}$. Taking a union bound, one can see that \[
\PP(A) \geq 1- \delta. 
\]

The subsequent reasoning conditions on $A$. 
$\forall f \in \cF$, we have
    \begin{align*}
  &L(f^*) - L(f)
  \\=& \max_{h\in \cF, \phi' \in [0,\war\phi], c'\in [0,\fr14]} \min_{h'\in \cF, \phi \in [0,\war\phi], c\in [0,\fr14]} ~~ \fr1n H_{\phi',c'}(h,f^*) - \fr1n H_{\phi,c}(h',f) 
  \tag{definition of $L$}
  \\ \le& \max_{h\in \cF, \phi' \in [0,\war\phi], c'\in [0,\fr14]} \min_{\phi \in [0,\war\phi], c\in [0,\fr14]} ~~ \fr1n H_{\phi',c'}(h,f^*) - \fr1n H_{\phi,c}(f^*,f) 
  \tag{$f^* \in \cF$}
\\ \le& \max_{h\in \cF, \phi' \in [0,\war\phi], c'\in [0,\fr14]} \min_{\phi \in [0,\war\phi], c\in [0,\fr14]} ~~ \fr1n\ln\del{\fr{16 |\cF| \war\phi n^2}{\dt}} 
  \\&\qquad +  \EE_x\sbr[2]{ -\fr{\Dt_x\warDt_{f^*,x,\phi,c}}{1 + \Dt_x\warDt_{f^*,x,\phi,c}}  + \fr 43\cd\sig_x^2 \warDt^2_{f^*,x,\phi,c} } + \fr1n\ln\del{\fr{48 |\cF| \war\phi n^2}{\dt}}
\tag{definition of $A_1, A_2$}
\\=& \min_{\phi \in [0,\war\phi], c\in [0,\fr14]} ~~ \fr1n\ln\del[2]{\fr{16 |\cF| \war\phi n^2}{\dt}} +  \EE_x\sbr[2]{ -\fr{\Dt_x\warDt_{f^*,x,\phi,c}}{1 + \Dt_x\warDt_{f^*,x,\phi,c}}  + \fr 43\cd\sig_x^2 \warDt^2_{f^*,x,\phi,c} } + \fr1n\ln\del[2]{\fr{48 |\cF| \war\phi n^2}{\dt}}
\\ \le& \min_{\phi \in [0,\war\phi], c\in [0,\fr14]} ~~ \EE_x\sbr[2]{ -\fr{\Dt_x\warDt_{f^*,x,\phi,c}}{1 + \Dt_x\warDt_{f^*,x,\phi,c}}  + \fr 43\cd\sig_x^2 \warDt^2_{f^*,x,\phi,c} } + \fr2n\ln\del{\fr{48 |\cF| \war\phi n^2}{\dt}}
\end{align*} 
That is, 
\begin{align*}
 \max_{\phi \in [0,\war\phi]} \max_{c\in[0,\fr 14]} \underbrace{ \EE_x\sbr[2]{ \fr{\Dt_x\warDt_{f^*,x,\phi,c}}{1 + \Dt_x\warDt_{f^*,x,\phi,c}}  - \fr 43\cd\sig_x^2 \warDt^2_{f^*,x,\phi,c} } }_{\tsty =: \LHS} \le \fr2n\ln\del{\fr{48 |\cF| \war\phi n^2}{\dt}} + (L(f) - L(f^*) )
\end{align*}
Recall that $\Dt_x = f^*_x - f_x$ and $\warDt_{f^*,x,\phi,c}= \clip{\phi(f^*_x-f_x)}_{[-c, c]}$. By Lemma~\ref{lem:basic-properties-U}, $\Dt_x\warDt_{f^*,x,\phi,c} \geq 0$ for all $x$ and $U_x \leq \fr 14$. 


Therefore, 
\begin{align*}
    \LHS
    =& 
    \EE_x\sbr[2]{ \fr{\Dt_x\warDt_{f^*,x,\phi,c}}{1 + \Dt_x\warDt_{f^*,x,\phi,c}}  - \fr 43\cd\sig_x^2 \warDt^2_{f^*,x,\phi,c} }
    \\\ge& \EE_x\sbr[2]{ \fr 45 \Dt_x\warDt_{f^*,x,\phi,c}  - \fr 43\cd\sig_x^2 \warDt^2_{f^*,x,\phi,c} }
    \tag{$\Dt_x\warDt_{f^*,x,\phi,c} \geq 0, ~|\warDt_{f^*,x,\phi,c}| \leq \fr 14, |\Dt_x| \leq 1$}
    \\=& \EE_x\sbr[2]{\fr 45 |\Dt_x|\del{\phi |f^* - f| \wed c  } - \fr 43\cd\sig_x^2 \del{\phi |f^* - f| \wed c  }^2}
    \tag{$\warDt_{f^*,x,\phi,c} = \sign(f^* - f) \del{\phi |f^* - f| \wed c  }$}
    \\=& \fr 45  \EE_x\sbr[2]{|\Dt_x|^2 \del{ \phi  \wed \fr{c}{|\Dt_x|} } \sbr{1- \fr 53 \cd\sig_x^2 \del{\phi  \wed \fr{c}{|\Dt_x| }} } }
\end{align*}
We want to set $c$ and $\phi$ such that
\begin{align}\label{eq:25-0512-target-ineq-2}
  \EE \fr12 |\Dt_x|^2 \del{ \phi  \wed \fr{c}{|\Dt_x|} } 
  \ge \EE \fr 53\cd |\Dt_x|^2 \sig_x^2 \del{ \phi  \wed \fr{c}{|\Dt_x|} }^2, 
\end{align}
which will give us the inequality of
\begin{align}\label{eq:25-0512-target-ineq-1}
  \fr 45 \EE \fr12 |\Dt_x|^2 \del{ \phi  \wed \fr{c}{|\Dt_x|} }  \le \fr2n\ln\del{\fr{48 |\cF| \war\phi n^2}{\dt}} + (L(f) - L(f^*) ).
\end{align}
We choose $\phi$ such that $\phi= \fr{c}{\Dt^*}$ for some $\Dt^*$ to be chosen later, we can see that $\phi\wed \fr{c}{|\Dt_x|} = c \del{\fr{1}{\Dt^*} \wed \fr{1}{|\Dt_x|} }$.
Using this, the above inequality~\eqref{eq:25-0512-target-ineq-2} becomes:
\begin{align*}
  \EE \fr12 |\Dt_x|^2 c \del{ \fr{1}{\Dt^*} \wed \fr{1}{|\Dt_x|} } 
  \ge \EE \fr 53\cd |\Dt_x|^2 c^2 \sig_{x}^2 \del{ \fr{1}{\Dt^*} \wed \fr{1}{|\Dt_x|} }^2
\end{align*}
We choose $c := c_0 \wed \fr 14$, where
\begin{align*}
  c_0 = \fr{ \EE \fr12 |\Dt_x|^2 \del{ \fr{1}{\Dt^*} \wed \fr{1}{|\Dt_x|} }  }{  \EE \fr 53\cd |\Dt_x|^2 \sig_{x}^2 \del{ \fr{1}{\Dt^*} \wed \fr{1}{|\Dt_x|} }^2 } 
\end{align*}
\begin{itemize}
    \item If $c_0 \leq \fr 14$, then $c = c_0$. 
    Plugging this into~\eqref{eq:25-0512-target-ineq-1} along with the fact $\phi\wed \fr{c}{|\Dt_x|} = c \del{\fr{1}{\Dt^*} \wed \fr{1}{|\Dt_x|} }$, we have
    \begin{align*}
      &\fr 45 \sbr[2]{ \EE \fr12 |\Dt_x|^2 c \del[2]{ \fr{1}{\Dt^*} \wed \fr{1}{|\Dt_x|} }  }^2
      \\&\le \EE \fr 53\cd |\Dt_x|^2 \sig_{x}^2 \del{ \fr{1}{\Dt^*} \wed \fr{1}{|\Dt_x|} }^2 \cd \del{\fr2n\ln\del{\fr{48 |\cF| \war\phi n^2}{\dt}} + (L(f) - L(f^*) )} 
      \\&\le \EE \fr 53 \sig_{x}^2 \cd \del{\fr2n\ln\del{\fr{48 |\cF| \war\phi n^2}{\dt}} + (L(f) - L(f^*) )}    
      \\ \implies
      &\sbr{\EE |\Dt_x|^2 \del{ \fr{1}{\Dt^*} \wed \fr{1}{|\Dt_x|} }  }^2 
      \\&\le \fr{25}{12}  \EE \sig_{x}^2 \cd \del{ \fr2n\ln\del{\fr{48 |\cF| \war\phi n^2}{\dt}} + (L(f) - L(f^*) )}
    \end{align*}
    We could lower bound the LHS above by picking out the region with $|\Dt_x| \geq \Dt^*$ to arrive at:
    \begin{align*}
      \EE \onec{|\Dt_x| \geq \Dt^*}  |\Dt_x|     
      &\le \sqrt{\fr{25}{12}  \EE \sig_{x}^2 \cd \del {\fr2n\ln\del{\fr{48 |\cF| \war\phi n^2}{\dt}}  + (L(f) - L(f^*) )}}
    \end{align*}

    \item If $c_0 > \fr 14$, then $c = \fr 14$. $c = \fr 14 < c_0$ implies that~\eqref{eq:25-0512-target-ineq-2} is true.  
    
    Plugging $c = \fr 14 $ into~\eqref{eq:25-0512-target-ineq-1} along with the fact $\phi \wed \fr{c}{|\Dt_x|} = c \del{\fr{1}{\Dt^*} \wed \fr{1}{|\Dt_x|} }$, we have
    \begin{align*}
     \fr 25 \EE |\Dt_x|^2 \del{ \fr{1}{\Dt^*} \wed \fr{1}{|\Dt_x|} }    
      &\le   \fr2n\ln\del{\fr{48 |\cF| \war\phi n^2}{\dt}} + (L(f) - L(f^*) )
    \end{align*}
    We could lower bound the LHS above by picking out the region with $|\Dt_x| \geq \Dt^*$ to arrive at:
    \begin{align*}
      \fr 25 \EE \onec{|\Dt_x| \geq \Dt^*}  |\Dt_x|    
      &\le  \fr2n\ln\del{\fr{48 |\cF| \war\phi n^2}{\dt}} + (L(f) - L(f^*) )
        \\ \implies
      \EE \onec{|\Dt_x| \geq \Dt^*}  |\Dt_x|   
      &\le 5  \fr1n\ln\del{\fr{48 |\cF| \war\phi n^2}{\dt}} +  \fr 52(L(f) - L(f^*) )
    \end{align*}
\end{itemize}
In either case, we have: 
\begin{align*}
    \EE \onec{|\Dt_x| \geq \Dt^*}  |\Dt_x| 
    &\le  \sqrt{  \fr{25}{12}  \EE \sig_{x}^2 \cd \del{\fr2n\ln\del{\fr{48 |\cF| \war\phi n^2}{\dt}} + (L(f) - L(f^*) ) } } 
    \\&+  \fr5n\ln\del{\fr{48 |\cF| \war\phi n^2}{\dt}} +  \fr 52(L(f) - L(f^*) )
\end{align*}

We choose $\Dt^* = \fr1n\ln\del{\fr{48 |\cF| \war\phi n^2}{\dt}}$, which gives us,
\begin{align*}
  \EE \onec{|\Dt_x| < \Dt^*}  |\Dt_x|   
  &\le \fr1n\ln\del{\fr{48 |\cF| \war\phi n^2}{\dt}}
\end{align*}

Altogether, we have, 
\begin{align*}
  &\EE_x |\Dt_x|   
  \\&\le \sqrt{  \fr{25}{12}  \EE \sig_{x}^2 \cd \del[2]{\fr2n\ln\del[2]{\fr{48 |\cF| \war\phi n^2}{\dt}} + (L(f) - L(f^*) ) } } +  \fr6n\ln\del{\fr{48 |\cF| \war\phi n^2}{\dt}} +  \fr 52(L(f) - L(f^*) )
\end{align*}
We verify the choice $\phi$ is valid as follows. 
\begin{align*}
    \phi = \fr{c}{\Dt^*} 
    \leq& \fr{1}{4\Dt^*}
    \tag{$c \leq \fr 14$}
    \\ =& \fr{1}{ 4 \fr1n\ln\del{\fr{48 |\cF| \war\phi n^2}{\dt}}}
    \\ =&  \fr{1}{ 4 \fr1n\ln\del{\fr{12 |\cF| n^3}{\dt}}} 
    \tag{$\war\phi = \fr n4$}
    \\ \leq& \war\phi
\end{align*}
which validates that $\phi \in [0, \war \phi]$. 

\end{proof}

\section{Proof of Corollary~\ref{cor:linear-class}}

\begin{lemma} \label{lem:Lipschitz-parameter-of-betting-loss}
    Recall that
\begin{align*}
  L(f) := \max_{h \in \cF} \max_{\phi \in [0,\war\phi]} \max_{c \in [0,\fr14]} \fr 1n \sum_{(x,y) \in D_n} \ln \del{ 1+ (y-f_x) \clip{\phi(h_x-f_x)}_{[-c, c]} } 
\end{align*}
$L$ is $\fr 43 n$-Lipchitz w.r.t. $\|\cd\|_\infty$.
\end{lemma}
\begin{proof}

    For fixed $(h, \phi, c)$, define: \[
    \Phi(f, h, \phi,c) := 
    \fr 1n \sum_{(x,y) \in D_n} \ln \del{ 1+ (y-f_x) \clip{\phi(h_x-f_x)}_{[-c, c]} }. 
    \]
    We first show that $\Phi(f, h, \phi,c)$ is Lipschitz in $f$ w.r.t. $\|\cd\|_\infty$. 

    Let $\varphi_1(t):= (y-t) \clip{\phi(h_x-t)}_{[-c, c]}$ for $t \in [0,1]$ and $\varphi_2(t) := \ln(1+t)$ for $t \in [-1/4,1/4]$. If $\phi(h_x-t) \in [-c,c]$, then $|\varphi_1'(t)| = \phi |(t-y)+(t-h_x)| \leq \phi + c \leq n $; else if $\phi(h_x-t) \not\in [-c,c]$, then $|\varphi_1'(t)| = c \leq \fr 14$. Hence $\varphi_1$ is $n$-Lipschitz. $|\varphi_2'(t)| = \fr{1}{1+t} \leq \fr 43$. Therefore, for any $(h, \phi, c)$, $\Phi(f, h, \phi,c)$ is $\fr 43 n$-Lipchitz in $f$ w.r.t. $\|\cd\|_\infty$: \[
    \forall f, f' \in \cF, ~~
    |\Phi(f,h, \phi,c)- \Phi(f',h, \phi,c)| 
    \leq \fr 43 n \cd \|f-f'\|_\infty. 
    \]
    Furthermore, $\forall f, f' \in \cF$, 
    \begin{align*}
        L(f) - L(f') 
        &= \max_{g \in \warcF(f), c \in [0,\fr14]} \Phi(f,g,c) - \max_{g \in \warcF(h), c \in [0,\fr14]} \Phi(f,g,c)
        \\&\le \max_{g \in \warcF(f), c \in [0,\fr14]} \Phi(f,h, \phi,c)- \Phi(f',h, \phi,c)
        \\&\leq \fr 43 n \cd \|f-f'\|_\infty
    \end{align*}
    By symmetry, 
    \begin{align*}
        L(f') - L(f) \leq \fr 43 n \cd \|f-f'\|_\infty. 
    \end{align*}
    Therefore, $\forall f, f' \in \cF$, \[
    |L(f) - L(f')| \leq \fr 43 n \cd \|f-f'\|_\infty. 
    \]
\end{proof}

\begin{corollary}[Linear class. Restatement of Corollary~\ref{cor:linear-class}]
Let $\cF$ be a linear function class in $d$-dimensional space: $\cF = \{x \mapsto x^\T\theta+\fr12: \normz{\theta}_2 \le \fr12\}$ and the instance space $\cX = \{x\in \RR^d: \normz{x}_2 \le 1\}$. Then,
with probability at least $1-\dt$, the output $\hf$ of Algorithm~\ref{alg:betting-regression} satisfies: 
\begin{align*}
    \EE_x |\hf_x - f^*_x|
    \le& \sqrt{ \fr{25}{3} \EE \sig_{x}^2 \fr dn \ln (\fr{48 n^5}{ \dt})  } + 12 \fr dn \ln (\fr{48 n^5}{ \dt}) 
\end{align*}
\end{corollary}
\begin{proof}

Let $\cF_\eps$ be a minimum $\eps$-cover of $\cF$ w.r.t. the metric $\|\cd\|_\infty$. Then, we can designate $\hf^\eps \in \cF_\eps$ such that $\normz{\hf^\eps - \hf}_\infty \le \eps$.

Applying Theorem~\ref{thm:second-order-main-theorem} with $\cF' \larrow \cF_\eps \cup \{f^*\}$, we have
\begin{align*}
  \EE_x |\hf^\eps_x - f^*_x| 
  \leq \sqrt{  \fr{25}{12}  \EE \sig_{x}^2 \cd \del{2\fr{L}n + (L^\eps(\hf^\eps) - L^\eps(f^*) )}  } +   6\fr{L}n + \fr 52(L^\eps(\hf^\eps) - L^\eps(f^*) ), 
\end{align*}
where 
\begin{align*}
&L= \ln\del{\fr{48 |\cF'| \war\phi n^2}{\dt}} 
\\&L^\eps(f) := \max_{h \in \cF'} \max_{\phi \in [0,\war\phi]} \max_{c \in [0,\fr14]} \fr 1n \sum_{(x,y) \in D_n} \ln \del{ 1+ (y-f_x) \clip{h_x-f_x}_{[-c, c]} } 
\end{align*}
For analysis purposes, we also denote:
\begin{align*}
    L(f) := \max_{h \in \cF} \max_{\phi \in [0,\war\phi]} \max_{c \in [0,\fr14]} \fr 1n \sum_{(x,y) \in D_n} \ln \del{ 1+ (y-f_x) \clip{h_x-f_x}_{[-c, c]} } 
\end{align*}
Recall that our Algorithm~\ref{alg:betting-regression} returns $\hf \in \argmin_{f \in \cF} L(f)$. We have, 
\begin{align}
    L^\eps(\hf^\eps) - L^\eps(f^*) 
    &= \del{L^\eps(\hf^\eps) - L^\eps(\hf)} + \del{L^\eps(\hf) -L(\hf)} +\del{L(\hf) -L(f^*)} +\del{L(f^*)- L^\eps(f^*)} \nonumber 
    \\&\le \fr 43 n \cd \normz{\hf^\eps - \hf}_{\infty} + 0 + 0 + \del{L(f^*)- L^\eps(f^*)}
\label{eqn:L-epsilon-difference-1}
\end{align}
where the first term is by Lemma~\ref{lem:Lipschitz-parameter-of-betting-loss}, the second term is by the definition of $L^\eps$ and $L$, and the third term is by the definition of $\hf$. 

We bound $L(f^*)- L^\eps(f^*)$ as follows. Let $h^*$ be such that \[
L(f^*) = \max_{\phi \in [0,\war\phi]} \max_{c \in [0,\fr14]} \fr 1n \sum_{(x,y) \in D_n} \ln \del{ 1+ (y-f^*_x) \clip{h^*_x-f^*_x}_{[-c, c]} } .
\] 
Let $h^*_\eps \in \cF_\eps$ be such that $\|h^*-h^*_\eps\| \leq \eps$. Then, 
\begin{align*}
    &L(f^*)- L^\eps(f^*) 
    \\\leq& \max_{\phi \in [0,\war\phi]} \max_{c \in [0,\fr14]} \fr 1n \sum_{(x,y) \in D_n} \ln \del{ 1+ (y-f^*_x) \clip{\phi^*(h^*_x-f^*_x)}_{[-c, c]} } 
    \\&
    - \max_{\phi \in [0,\war\phi]} \max_{c \in [0,\fr14]} \fr 1n \sum_{(x,y) \in D_n} \ln \del{ 1+ (y-f^*_x) \clip{\phi^*(h^*_{\eps,x}-f^*_x)}_{[-c, c]} }
\end{align*}
Note that for any $(\phi, c)$, 
\begin{align*}
    &\abs{(y-f^*)\clip{\phi\cd (h^*-f^*)}_{[-c,c]} - (y-f^*)\clip{\phi\cd (h^*_\eps-f^*)}_{[-c,c]} }
    \\&= |(y-f^*)| \cd \abs{\clip{\phi\cd (h^*-f^*)}_{[-c,c]} - \clip{\phi\cd (h^*_\eps-f^*)}_{[-c,c]}}
    \\&\le |(y-f^*)| \cd \fr n4\eps
    \tag{$\phi \in [0, \fr n4]$}
\end{align*}
Using $t \mapsto \ln(1+t)$ is $\fr 43$-Lipschitz for $t \in [-\fr 14, \fr 14]$, 
\begin{align*}
    L(f^*)- L^\eps(f^*) 
    &\leq \fr 43 \cd \fr n4\eps \cd \fr1n \sum_{(x,y)}|(y-f^*_x)|
    \\&\leq \fr 13 n\eps. 
\end{align*}
Plugging back into Eqn.~\eqref{eqn:L-epsilon-difference-1}, 
\begin{align}
L^\eps(\hf^\eps) - L^\eps(f^*) 
\leq 
\fr 43 n \cd \normz{\hf^\eps - \hf}_{\infty} + \del{L(f^*)- L^\eps(f^*)}
\leq 2n\eps
\label{eqn:L-epsilon-difference-2}
\end{align}

Therefore, 
\begin{align*}
      &\EE_x |\hf_x - f^*_x| 
\\&\le \EE_x |\hf_x - \hf^\eps_x| + \EE_x |\hf^\eps_x - f^*_x| 
\tag{Triangle inequality}
\\&\le \eps +  \sqrt{  \fr{25}{12}  \EE \sig_{x}^2 \cd \del{2\fr{L}n + (L^\eps(\hf^\eps) - L^\eps(f^*) )}  } +   6\fr{L}n + \fr 52(L^\eps(\hf^\eps) - L^\eps(f^*) )
\tag{Theorem~\ref{thm:second-order-main-theorem}}
\\&\le \eps +  \sqrt{  \fr{25}{12}  \EE \sig_{x}^2 \cd \del{2\fr{L}n + 2 n\eps}  } +   6\fr{L}n + \fr 52(2 n\eps )
\tag{Eqn.~\eqref{eqn:L-epsilon-difference-2}}
\\&\le \sqrt{  \fr{25}{12} \EE \sig_{x}^2 (\fr{2}{n}\ln(\fr{(1 + (3/\eps)^d)48 \war\phi n^2}{\dt}) +2n \eps)} + 6\fr{1}{n}\ln(\fr{(1 + (3/\eps)^d)48\war\phi n^2}{\dt}) +  6n \eps 
\end{align*}

where the last inequality is by the covering number of the linear class (e.g., Exercise 20.3 of~\citet{lattimore18bandit}). 

Choosing $\eps = \fr 1{n^2}$ gives: 
\begin{align*}
    \EE_x |\hf_x - f^*_x|
    \le& \sqrt{ \fr{25}{3} \EE \sig_{x}^2 \fr dn \ln (\fr{48 n^5}{ \dt})  } + 12 \fr dn \ln (\fr{48 n^5}{ \dt}) 
\end{align*}
\end{proof}

\section{Comparing the Two First-Order Quantities}
\label{sec:comparison-first-order}
In this section, we first show that $\EE_x[f^*(x) \wedge (1-f^*(x))] \leq \EE_x[f^*(x)] \wed \EE_x[1-f^*(x)]$, then we give an example where the difference between these two quantities can be arbitrarily large. 
\begin{lemma}
    Recall that $f^*: \cX \rightarrow [0,1]$. We have, \[
    \EE_x[f^*(x) \wedge (1-f^*(x))] \leq \EE_x[f^*(x)] \wed \EE_x[1-f^*(x)]
    \]
\end{lemma}
\begin{proof}
    Note that 
    \begin{align*}
        \EE_x[f^*(x) \wedge (1-f^*(x))] \leq \EE_x[f^*(x)],
    \end{align*}
    and 
    \begin{align*}
        \EE_x[f^*(x) \wedge (1-f^*(x))] \leq \EE_x[1-f^*(x)].
    \end{align*}
    Hence, 
    \[
    \EE_x[f^*(x) \wedge (1-f^*(x))] \leq \EE_x[f^*(x)] \wed \EE_x[1-f^*(x)]. 
    \]
\end{proof}
\begin{example}
    Let $\eps>0$ be a small number, $Y$ be of the distribution $\PP(Y = \eps) = \PP(Y = 1-\eps) = \fr 12$. Then, \[
    \EE[Y \wed 1-Y] = \eps,
    \] 
    whereas \[
    \EE [Y] \wed \EE[1-Y] = \fr 12. 
    \]
\end{example}
\section{Proof of Lemma~\ref{lem:var-ub}}
\begin{proof}
\begin{align*}
    \Var(Y) 
    &= \EE [(Y-\EE [Y])^2 ]
    \\&= \EE [Y^2 - 2 Y \EE [Y] + \EE^2 [Y] ]
    \\&\le \EE [Y - 2 Y \EE [Y] + \EE^2 [Y] ]
    \tag{$Y \in [0,1]$}
    \\&= \EE [Y] - \EE[Y^2]
    \\&=\EE[Y(1-Y)]
\end{align*}
One can see that the equality in the third line is attained iff $Y$ is Bernoulli distributed. 
\end{proof}



\end{document}